%% file: paper_arxiv.tex
\documentclass[singlecolumn]{article}

\usepackage{geometry}
\usepackage{microtype}
\usepackage{graphicx}
\usepackage{multirow}
\usepackage{booktabs} % for professional tables

\usepackage{dsfont}
\usepackage{enumitem}
\usepackage{graphicx}
\usepackage{wrapfig}  % Allow wrapping of text around figures
\usepackage{subfig}
\usepackage{multirow}
\usepackage{bm,amsbsy} % for bold symbols
\usepackage{amsfonts} % various math stuff
\usepackage{amsmath}
\usepackage{amssymb}
\usepackage{algorithm,algorithmic}
\usepackage{verbatim}
\usepackage{bbm}
\usepackage{amsthm,mathtools}

\usepackage{booktabs} % for professional tables
\usepackage{xcolor}
\usepackage{color}
\usepackage{hyperref}
\usepackage{thm-restate}
\usepackage{etoolbox}

\usepackage{wrapfig}  % Allow wrapping of text around figures
\usepackage{mdwlist}  % Make list items closer together: itemize*, enumerate*
\usepackage{sidecap}  % For putting caption beside figure
\usepackage{caption}
\newcommand{\wjsay}[1]{}

\newcommand{\mvsay}[1]{}

\newcommand{\ourtitle}{ABCDP: Approximate Bayesian Computation\\with Differential Privacy}
\allowdisplaybreaks
\title{\ourtitle}
\author{Mijung Park$^{1,2}$ \hspace{8mm} Margarita Vinaroz$^{1,2}$ \hspace{8mm} Wittawat Jitkrittum$^{1,3}$ \\[3mm]
    \url{mijung.park@tuebingen.mpg.de} \\
    \url{mvinaroz@tuebingen.mpg.de} \\
    \url{wittawatj@gmail.com} \\[6mm]
$^1$Max Planck Institute for Intelligent Systems, T\"{u}bingen, Germany \\[1mm]
$^2$Department of Computer Science, University of T\"{u}bingen  \\[1mm]
$^3$Google Research  }

\begin{document}
\raggedbottom
\maketitle

\input{notation}

\begin{abstract}
 We develop a novel approximate Bayesian computation (ABC) framework, \textit{ABCDP}, that produces differentially private (DP) and approximate posterior samples. Our framework takes advantage of the Sparse Vector Technique (SVT), widely studied in the differential privacy literature.  SVT incurs the privacy cost only when a condition (whether a quantity of interest is above/below a threshold) is met. If the condition is met sparsely during the repeated queries, SVT can drastically reduces the cumulative privacy loss, unlike the usual case where  every query incurs the privacy loss. In ABC, the quantity of interest is the distance between observed and simulated data, and only when the distance is below a threshold, we take the corresponding prior sample as a posterior sample. Hence, applying SVT to ABC is an organic way to transform an ABC algorithm to a privacy-preserving variant with minimal modification, but yields the posterior samples with a high privacy level. We theoretically analyze the interplay between the noise added for privacy and the accuracy of the posterior samples.
%  We consider both Laplace noise with the traditional linear composition and Gaussian noise with a more recent R\'enyi-DP composition.
We apply ABCDP to several data simulators and show the efficacy of the proposed framework.
\end{abstract}

\section{Introduction}

% (1) generate statements about ABC, why people use it, then conclude with privacy aspect

Approximate Bayesian computation (ABC) aims at identifying the posterior distribution over simulator parameters.
The posterior distribution is of interest as it provides the mechanistic understanding of the stochastic procedure that directly generates data
in many areas such as climate and weather, ecology, cosmology, and bioinformatics \cite{Tavare1997, Ratmann_07, Bazinetal_2010, Schaferetal_12}.
Under these complex models, directly evaluating the likelihood of data given the parameters is often intractable.
ABC resorts to an approximation of the likelihood function using simulated data that are {\it{similar}} to the actual observations.

In the simplest form of ABC called \textit{rejection ABC} \cite{rejABC}, we  proceed by
sampling multiple model parameters from a prior distribution $\pi$: $\theta_1, \theta_2, \ldots \sim \pi$. For each $\theta_{t}$, a
pseudo dataset $Y_{t}$ is generated from a simulator (the forward sampler associated with the intractable likelihood $p(y|\theta)$). The parameter $\theta_{t}$
for which the generated $Y_{t}$ are similar to the observed $Y^*$, as decided by
$\rho(Y_{i}, Y^*) <
\epsilon_{{abc}}$,  are accepted. Here $\rho$ is a notion of distance, for instance, L2 distance beween $Y_t$ and $Y^*$ in terms of a pre-chosen summary statistic. Whether the distance is small or large is determined by $\epsilon_{{abc}}$, a \emph{similarity threshold}.
The result  is samples $\left\{ \theta_{t}\right\} _{t=1}^{M}$ from
a distribution, $\tilde{p}_{\epsilon}(\theta|Y^{*}) \propto
\pi(\theta)\tilde{p}_{\epsilon}(Y^{*}|\theta)$, where
$\tilde{p}_{\epsilon}(Y^{*}|\theta) = \int_{B_{\epsilon}(Y^{*})}p(Y|\theta)dY$
and $B_{\epsilon}\left(Y^{*}\right)  =  \left\{
    Y\;:\;\rho(Y,Y^{*})<\epsilon_{{abc}}\right\}$.
    As the likelihood computation is approximate, so is the posterior
    distribution. Hence, this framework is named by \textit{approximate} Bayesian
    computation, as we do not compute the likelihood of data explicitly.

% (2) why privacy matters in ABC inference, how to tackle it?
Most ABC algorithms evaluate the data similarity in terms of summary statistics computed by an aggregation of individual datapoints \cite{Joyce_Marjoram_08, Robert11, Nunes2010, Aeschbacher12, drovandi2015, Aeschbacher12, Fearnhead2012}. However,
this seemingly innocuous step of similarity check could impose a privacy threat, as aggregated statistics could still reveal an individual's participation to the dataset
with the help of combining other publicly available datasets (see \cite{Homer2008, Johnson2013}).
In addition, in some studies, the actual observations are privacy-sensitive in nature e.g., Genotype data for estimating Tuberculosis transmission parameters \cite{Tanaka:2006aa}.
Hence, it is necessary to privatize the step of similarity check in ABC algorithms.

% (3) What we are proposing here? What is our work significant?
In this light, we introduce an ABC framework that obeys the notion of {\it{differential privacy}}.
 The differential privacy definition provides a way to quantify the amount of information that the distance computed on the privacy-sensitive data contains on whether or not a single individual's data is included (or modified) in the data \cite{dwork2006calibrating}. Differential privacy also provides rigorous privacy guarantees in the presence of \textit{arbitrary side information} such as similar  public data available.

A common form of applying DP to an algorithm is by adding noise to outputs of
the algorithm, called \textit{output perturbation}\cite{Chaudhuri_11}.  In case of ABC, we found that \textit{adding noise to the distance}
computed on the real observations and pseudo-data suffices the privacy
guarantee of the resulting posterior samples.
% \mvsay{And adding noise to the threshold??}
However, if we choose to simply add noise to the distance in every ABC inference step, this DP-ABC inference imposes an additional challenge due to the \textit{repeated} use of the real observations.
%
%In terms of privacy, ABC inference imposes an additional challenge due to the ABC's iterative nature.
%An important property of DP, so-called \textit{composition theorem}
%
%states that privacy degrades as we use data repetitively.
%Hence, a \textit{repeated use of actual observations} to compare the similarity for a new candidate dataset in each ABC step will increase the privacy loss
%
The \textit{composition} property of differential privacy
 states  that the privacy level degrades over a repeated use of data. To overcome this challenge, we adopt the \textit{sparse vector technique} (SVT) \cite{Dwork14}, and apply it to the rejection ABC paradigm. The SVT outputs \textit{noisy} answers of whether or not a stream of queries being above a certain threshold, where the only place the privacy cost incurs is when SVT outputs at most $c$ “above threshold” answers.
 This is a significant saving of privacy cost, as arbitrarily many “below threshold” answers are privacy cost free.

 We name our framework, the ABC combined with SVT, as \textit{ABCDP} (Approximate Bayesian Computation with Differential Privacy).
Under ABCDP, we theoretically analyze the effect of noise added to the distance in the resulting posterior samples and the subsequent posterior integrals.
% what is intriguing is that ABC has a special parameter, the similarity threshold parameter, $\epsilon_{{abc}}$, and whether we could
% manipulate this parameter to change the final privacy level of the posterior samples. %
% What is the precise relationship between $\epsilon_{abc}$ and the final privacy level of the posterior samples?
% To date, the answer to this question is not well understood.
% In this paper, we formally provide a theoretical answer to this question.
%
%  a recently developed notion of privacy called \textit{Renyi differential privacy} (RDP) in order to obtain the smallest cumulative privacy loss after many repeated use of data \cite{46029, 2016arXiv160700133A, pmlr-v89-wang19b}.
%
Putting together, we conclude our introduction by summarizing our main contributions:
\begin{enumerate}[topsep=2pt,itemsep=2mm,parsep=1mm]
\item We provide a novel ABC framework, ABCDP, which combines \textit{sparse vector technique} (SVT) \cite{Dwork14} with the rejection ABC paradigm. The resulting ABCDP
    framework can improve the trade-off between privacy and accuracy of the
    posterior samples, as the privacy cost under ABCDP is a function of the number of \textit{only accepted} posterior samples.
\item We theoretically analyze ABCDP by focusing on the effect of noisy posterior samples in terms of two quantities. The first quantity provides the probability of an output of ABCDP being different from that of ABC at any given time during inference. The second quantity provides the convergence rate, i.e., how fast the posterior integral using ABCDP's noisy samples approaches that using non-private ABC's samples. We write both quantities as a function of added noise for privacy to better understand the characteristics of ABCDP.
% \mvsay{This isn't true anymore, we only consider Laplace noise, nevertheless we still can consider 2 cases with/without resampling the threshold}
\item We validate our theory in the experiments using several simulators.
% \item We also provide a framework for soft-thresholding-ABCDP. Due to the space limit, we present this framework in the Supplementary material.
\end{enumerate}

\section{Background}

We start by describing relevant background information.

\subsection{Approximate Bayesian Computation}
Given a set $Y^*$ containing observations,  \textbf{rejection ABC}
\cite{rejABC} yields samples from an approximate posterior
distribution by repeating the following three steps:
\begin{align}
\theta &\sim \pi(\theta), \\
Y=\{y_1, y_2, \ldots\}  &\sim p(y|\theta),\\
p_{\epsilon_{abc}}(\theta|Y^*) &\sim p_{\epsilon_{abc}}(Y^*|\theta)\pi(\theta),
\end{align}
where the pseudo dataset $Y$ is compared with the observations $Y^*$ via
\begin{align}\label{rej_ABC_similarity}
p_{\epsilon_{abc}}(Y^*|\theta) &= \int_{B_{\epsilon_{abc}}(Y^*)} p(Y|\theta) dY, \nonumber\\
B_{\epsilon_{abc}}(Y^*) & = \{Y| \rho(Y,Y^*) \leq \epsilon_{abc}\},
\end{align}
where $\rho$ is a divergence measure between two datasets.
Any distance metric can be used for $\rho$. For instance, one can use the L2 distance under two datasets in terms of a pre-chosen set of summary statistics,
i.e., $\rho(Y,Y^*) = D(S(Y), S(Y^*))$, with L2 distance measure $D$ on the statistics computed by $S$.

 A more statistically sound choice for $\rho$ would be
\textit{Maximum Mean Discrepancy} (MMD, \cite{Gretton2012}) as used in \cite{ParJitSej2016}.
Unlike a pre-chosen finite dimensional summary statistic typically used in ABC, MMD compares two distributions in terms of all the possible moments of the the random variables described by the two distributions. Hence,  ABC frameworks using the MMD metric such as \cite{ParJitSej2016} can avoid the problem of non-sufficiency of a chosen summary statistic that may incur in many ABC methods. For this reason, in this paper we demonstrate our algorithm using the MMD metric. However, other metrics can be used as we illustrated in our experiments.

\paragraph{Maximum Mean Discrepancy}

Assume that the data $Y\subset\mathcal{X}$
and let $k\colon\mathcal{X}\times\mathcal{X}$ be a positive definite
kernel. The MMD between two distributions $P,Q$ is defined as $ \mathrm{MMD}^2(P,Q)=
 \mathbb{E}_{x,x'\sim P}k(x,x')+\mathbb{E}_{y,y'\sim Q}k(y,y') -2\mathbb{E}_{x\sim P}\mathbb{E}_{y\sim Q}k(x,y). $
% The Moore--Aronszajn theorem states that there is a unique Hilbert
% space $\mathcal{H}$ on which $k$ defines an inner product. As a
% result, there exists a feature map $\phi\colon\mathcal{X}\to\mathcal{H}$
% such that $k(x,y)=\left\langle \phi(x),\phi(y)\right\rangle _{\mathcal{H}}$,
% where $\left\langle \cdot,\cdot\right\rangle _{\mathcal{H}}=\left\langle \cdot,\cdot\right\rangle $
% denotes the inner product on $\mathcal{H}$. The MMD in (\ref{eq:pop_mmd})
% can be written as
% \begin{align*}
% \mathrm{MMD}(P,Q) & =\big\|\mathbb{E}_{x\sim P}[\phi(x)]-\mathbb{E}_{y\sim Q}[\phi(y)]\big\|_{\mathcal{H}},
% \end{align*}
% %
% where $\mathbb{E}_{x\sim P}[\phi(x)]\in\mathcal{H}$ is known as the
% (kernel) mean embedding of $P$, and exists if $\mathbb{E}_{x\sim P}\sqrt{k(x,x)}<\infty$
% \cite{Smola2007}.
% The MMD can be interpreted as the distance
% between the mean embeddings of the two distributions.
If $k$ is a
characteristic kernel \cite{Sriperumbudur2011}, then $P\mapsto\mathbb{E}_{x\sim P}[\phi(x)]$
is injective, implying that $\mathrm{MMD}(P,Q)=0$, if and only if $P=Q$.
When $P,Q$ are
observed through samples $X_{m}=\{x_{i}\}_{i=1}^{m}\sim P$ and $Y_{n}=\{y_{i}\}_{i=1}^{n}$, MMD can be estimated
by empirical averages
\cite[eq.3]{Gretton2012}: $\widehat{\mathrm{MMD}}^2(X_{m},Y_{n}) = \tfrac{1}{m^2}\sum_{i,j=1}^{m}k(x_{i},x_{j}) + \tfrac{1}{n^2}\sum_{i,j=1}^{n}k(y_{i},y_{j}) -\tfrac{2}{mn}\sum_{i=1}^{m}\sum_{j=1}^{n}k(x_{i},y_{j}).$
% \begin{align}
% \widehat{\mathrm{MMD}}^2(X_{m},Y_{n}) &= \tfrac{1}{m^2}\sum_{i,j=1}^{m}k(x_{i},x_{j}) + \tfrac{1}{n^2}\sum_{i,j=1}^{n}k(y_{i},y_{j})\nonumber \\
% &\qquad -\tfrac{2}{mn}\sum_{i=1}^{m}\sum_{j=1}^{n}k(x_{i},y_{j}).
% \end{align}
When applied in the ABC setting, one input to $\widehat{\mathrm{MMD}}$ is the observed dataset
$Y^{*}$ and the other input is a pseudo dataset $Y_{t}\sim p(\cdot|\theta_{t})$ generated
by the simulator given  $\theta_{t}\sim\pi(\theta)$.
%
% Note that the total computational cost of the estimator $
% \widehat{\mathrm{MMD}}(X_{m},Y_{n}) $ is $O(Tmn)$ for the $T$ number of
% posterior samples.
% %
% To reduce the complexity, sub-quadratic time MMD estimators exist e.g., an
% unbiased linear-time estimator \cite[Section 6]{Gretton2012}. See
% \suppsecref{MMD_RFF} for MMD with random Fourier features \cite{rahimi2008random} for another
% linear-time estimator.

\subsection{Differential Privacy}
An output from an algorithm that takes in sensitive data as input will
naturally contain some information of the sensitive data $\mathcal{D}$. The
goal of differential privacy is to augment such an  algorithm so that useful
information about the population is retained, while sensitive information such as an individual's participation
in the dataset cannot be learned \cite{Dwork14}. A common way to achieve these two
seemingly paradoxical goals is by deliberately injecting a controlled-level of random noise
to the to-be-released quantity. The modified procedure, known as a DP mechanism, now
gives a stochastic output  due to the injected noise. In the DP framework,
higher level of noise provides stronger privacy guarantee at the expense of
less accurate population-level information that can be derived from the
released quantity. Less noise added to the output thus reveals more about an individual's presence
in the dataset.% A mechanism is called $\epsilon$-differentially private if

More formally, given a mechanism $\mathcal{M}$ (a \textit{mechanism} takes a dataset as input and produces stochastic outputs) and neighbouring datasets
$\Dat$, $\Dat'$ differing by a single entry (either by replacing one datapoint with another, or by adding/removing a datapoint to/from $\Dat$), the \emph{privacy loss} of an
outcome $o$ is defined by
\begin{equation}
L^{(o)} = \log \frac{P(\mathcal{M}(\Dat) = o)}{P(\mathcal{M}(\Dat') = o)}.
\end{equation}
The mechanism $\mathcal{M}$ is called $\epsilon$-DP if and only if
$|L^{(o)}|\leq \epsilon$,
for all possible outcomes $o$ and for all possible neighbouring datasets $\Dat, \Dat'$.
The definition states that a single individual's participation in the data does
not change the output probabilities by much; this limits the amount of
information that the algorithm reveals about any one individual.
A weaker or an \textit{approximate} version of the above notion is ($\epsilon, \delta$)-DP: $\mathcal{M}$
is ($\epsilon, \delta$)-DP if $|L^{(o)}| \leq \epsilon$, with probability  $1-\delta$, where $\delta$ is often called a failure probability which quantifies how often the DP guarantee of the mechanism fails.
%\wjsay{Can you be precise with $o$? For all? And where to put ``for all''?}
%
%
%\mj{We need to talk about properties of DP, then move to RDP definition. }

Output perturbation is a commonly used DP mechanism to ensure the outputs of an algorithm to be differentially private. Suppose a deterministic function $h: \Dat \mapsto \mathbb{R}^p$ computed on sensitive data $\Dat$ outputs a $p$-dimensional vector quantity. In order to make $h$ private, we can add noise to the output of $h$, where the level of noise is calibrated to the {\it{global sensitivity}}
\cite{dwork2006our}, $\Delta_h$, defined by the maximum difference in terms of some norm $||h(\Dat)-h(\Dat') ||$ for neighboring $\Dat$ and $\Dat'$ (i.e. differ by one data sample).
% For the Gaussian mechanism (Theorem 3.22 in \cite{Dwork14}),
%  the perturbed output is given by
% %
% \begin{equation*}
%     \tilde{h}(\Dat) = h(\Dat) + \Nrm(0, \sigma^2 \mathbf{I}_p).
% \end{equation*}
% %
% The perturbed function $\tilde{h}(\Dat) $ is then $(\epsilon, \delta)$-DP,
% where $\sigma \geq \Delta_h \sqrt{2\log(1.25/\delta)}/\epsilon$, for $\epsilon
% \in (0,1)$ and  $\Delta_h$ uses L2 norm, $||h(\Dat)-h(\Dat') ||_2$. (See the proof of  Theorem 3.22 in \cite{Dwork14} for more details).

There are two important properties of differential privacy.
First,  the
\textit{post-processing invariance} property \cite{dwork2006our} tells us that
the composition of any arbitrary data-independent mapping with an
$(\epsilon,\delta)$-DP algorithm is also $(\epsilon,\delta)$-DP.
Second, the \textit{composability} theorem \cite{dwork2006our} states that the strength
of privacy guarantee degrades with repeated use of DP-algorithms.
Formally, given an $\epsilon_1$-DP mechanism
$\mathcal{M}_1$ and an $\epsilon_2$-DP mechanism
$\mathcal{M}_2$, the mechanism $\mathcal{M}(\mathcal{D}) :=
(\mathcal{M}_1(\mathcal{D}), \mathcal{M}_2(\mathcal{D}))$ is
$(\epsilon_1+\epsilon_2)$-DP.
This composition is often-called \textit{linear} composition, under which the total privacy loss linearly increases with the number of repeated use of DP-algorithms.
The \textit{strong} composition \cite{Dwork14}[Theorem 3.20] improves the linear composition, while the resulting DP guarantee becomes weaker (i.e., approximate $(\epsilon,\delta)$-DP).
Recently, more refined methods further improves the privacy loss (e.g., \cite{46029}).

\subsection{AboveThreshold and Sparse Vector Technique}
One of the DP mechanisms we will utilize for making rejection ABC differentially private is \textit{AboveThreshold} and \textit{Sparse vector technique} (SVT) \cite{Dwork14}.
AboveThreshold outputs 1 when a query value exceeds a pre-defined threshold, or 0 otherwise. This resembles rejection ABC where the output is 1 when the distance is less than a chosen threshold. To ensure the output is differentially private, AboveThreshold adds noise to both the threshold and the query value. We take the same route as AboveThreshold to make our ABCDP outputs differentially private.
Sparse vector technique (SVT) consists of $c$ calls to AboveThreshold, where $c$ in our case determines how many posterior samples ABCDP releases.

Before presenting our ABCDP framework, we first describe the privacy setup we consider in this paper.

\section{Problem formulation}
%\subsection{The privacy setup we consider}
\label{sec:Problem Setup}

We assume a \textit{data owner} who owns  sensitive data $Y^*$ and is willing to contribute to the posterior inference.
We also assume a \textit{modeler} who aims to learn the posterior
distribution of the parameters of a simulator.
Our ABCDP algorithm proceeds with the two steps:
\begin{enumerate}[topsep=2pt,itemsep=2mm,parsep=1mm]
    \item \textit{Non-private step:} The modeler draws a parameter sample $\theta_t \sim \pi(\theta)$; then generates a pseudo-dataset $Y_{t}$, where  $Y_t \sim p(y|\theta_t)$ for $t=1,\cdots,T$ for a large $T$.
    We assume these parameter-pseudo-data pairs $\{\theta_t, Y_t\}_{t=1}^T$ are publicly available (even to an adversary).
\item \textit{Private step:} The data owner takes the whole sequence of parameter-pseudo-data pairs $\{(\theta_t, Y_t)\}_{t=1}^T$ and
runs our ABCDP algorithm in order to
output a set of \textit{differentially private} binary indicators determining whether or not to accept each $\theta_t$.
\end{enumerate}

Note that $T$ is the maximum number of parameter-pseudo-data pairs that are publicly available. We will run our algorithm for $T$ steps, while our algorithm can terminate as soon as we output the $c$ number of accepted posterior samples. So, generally $c \leq T$. The details are introduced next.
 %
%
% In either case, the data owner applies our proposed DP mechanism to the output (Section
% \ref{sec:proposed_rej_abc} for rejection ABC, and Section
% \ref{sec:proposed_soft_abc} for soft ABC) to guarantee the privacy on the
% sensitive data.

%Here is the summary of what we discussed on March 6, 2019. We talked essentially about how to make K2-ABC differentially private (although we won't call it K2-ABC).

\section{ABCDP}
%---------------------------------
% \subsection{Proposed Mechanism for Rejection ABC}
% \label{sec:proposed_rej_abc}
%

Recall that the only place where the real data $Y^*$ appears in the ABC algorithm is when we judge whether the simulated data is similar to the real data, i.e., as in \eqref{rej_ABC_similarity}. Our method hence adds noise to this step. In order to take advantage of the privacy analysis of SVT, we also add noise to the ABC threshold and to the ABC distance. Consequently, we introduce two perturbation steps.

Before we introduce them, we describe the global sensitivity of the distance, as this quantity tunes the amount of noise we will add in the two perturbation steps.
For $\rho(Y^*,Y)=\widehat{\mathrm{MMD}}(Y^*,Y)$ with a bounded kernel, then the sensitivity of the distance is
$\Delta_\rho = O(1/N)$ as shown in Lemma \ref{lem:deltammd}.

\begin{restatable}[$\Delta_\rho = O(1/N)$ for MMD]{lem}{deltammd}
    \label{lem:deltammd}
Assume that $Y^{*}$ and each pseudo dataset $Y_t$ are of the same
cardinality $N$. Set $\rho(Y^*,Y)=\widehat{\mathrm{MMD}}(Y^*,Y)$ with
a kernel $k$ bounded by $B_{k} >0$ i.e., $\sup_{x,y\in\mathcal{X}}k(x,y)\le B_{k} <\infty$.
Then,
\begin{align}
\sup_{(Y^*, Y^{*'}),Y}|\rho(Y^*,Y)-\rho(Y^{*'},Y)| \le \Delta_{\rho}:=\frac{2}{N}\sqrt{B_{k}} \nonumber
\end{align}
and $\sup_{Y^*,Y}\rho(Y^*,Y) \le 2\sqrt{B_k}.$
\end{restatable}
A proof is given in \suppsecref{proof_deltammd}.
For $\rho = \widehat{\mathrm{MMD}}$ using a Gaussian kernel,
$k(x,y) = \exp \left(-\frac{\|x-y\|^2}{2l^2} \right)$ where $l>0$ is the
bandwidth of the kernel, $B_k=1$ for any $l>0$.

Now we introduce the two perturbation steps used in our algorithm summarized in \algoref{ABCDP}.

\textbf{Noise for privatizing the ABC threshold:}
	\begin{align}\label{m_t}
        \hat\epsilon_{abc} &=
        \epsilon_{abc} + m_t
	\end{align}
    where $ m_t \sim \mbox{Lap}(b)$, i.e., drawn from the zero-mean Laplace distribution with a scale parameter $b$.
% $\sigma_1 = \Delta_\rho \sigma$
%

\textbf{Noise for privatizing the distance:}
	\begin{align}\label{nu_t}
        \hat{\rho}_t &=
        \rho(Y^*, Y_t) + \nu_t
    \end{align}
    where $ \nu_t \sim \mbox{Lap}(2b)$.

Due to these perturbations, \algoref{ABCDP} runs with the privatized threshold and distance. We can choose to perturb the threshold only once, or every time we output $1$ by setting RESAMPLE either to false or true. After outputting $c$ number of $1$'s, the algorithm is terminated. How do we calculate the resulting privacy loss under different options we choose?
% \eqref{m_t}

\begin{algorithm}[t]
	\caption{Proposed $c$-sample ABCDP}
    \label{algo:ABCDP}
    \begin{algorithmic}[1]
		\vspace{0.1cm}
        \REQUIRE Observations $Y^*$, Number of accepted posterior sample size $c$, privacy tolerance
        $\epsilon_{total}$, ABC threshold
        $\epsilon_{abc}$,  distance $\rho$, and parameter-pseudo-data pairs $\{(\bm \theta_t, Y_t )\}_{t=1}^T$, and option RESAMPLE.
		\vspace{0.1cm}
        \ENSURE
 $\epsilon_{total}$-DP indicators  $\{\tilde\tau_t\}_{t=1}^{T}$ for corresponding samples $\{ \bm \theta_{t}\}_{t=1}^T$\\
        \vspace{2mm}
        %---- start the algorithm ---
        % \STATE Set $\sigma_1 =  \Delta_{\rho}\sigma$
        \STATE Calculate the noise scale $b$ by \thmref{abcdp_Lap}.
        \STATE Privatize ABC threshold: $\hat{\epsilon}_{abc} = \epsilon_{abc} + m_t$ via \eqref{m_t}
        \STATE Set count=0

        \FOR{$t = 1,\ldots, T$}
	\STATE Privatize distance: $\hat\rho_t = \rho(Y^*, Y_t) + \nu_t$ via \eqref{nu_t}
	\vspace{2pt}

        \IF { $ \hat\rho_t \le \hat{ \epsilon}_{abc}$ }
        \STATE Output $\tilde\tau_t= 1 $
        \STATE count $=$ count$ + 1$
        \IF {RESAMPLE}
        \STATE $\hat{\epsilon}_{abc} = \epsilon_{abc}+ m_t$ via \eqref{m_t}
        \ENDIF
        \ELSE
         \STATE Output $\tilde\tau_t= 0 $
        %  Collect $\tilde\tau_t= 0 $
        \ENDIF
        \IF {count $\geq$ c}
            \STATE Break the loop
        \ENDIF
        \ENDFOR
        % \STATE Return collected $\{\tilde{\tau}_t \}_{t=1}^T$
  \end{algorithmic}
%   \wjsay{Update the algorithm to include the Gaussian noise case.
%   Use $m$ for the noise added to the threshold. Use $\nu$ for the noise added to $\rho_t$ (MMD). Do not use $\rho$ for noise. We use it for MMD.}
\end{algorithm}

We formally state the relationship between the noise scale and the final privacy loss $\epsilon_{tot}$ for the Laplace noise
in \thmref{abcdp_Lap}.
\begin{restatable}[\algoref{ABCDP} is $\epsilon_{total}$-DP]{thm}{mrejdp}
    \label{thm:abcdp_Lap}
For any neighbouring datasets $Y^*,Y^{*'}$ of size $N$ and
any dataset $Y$, assume that $\rho$ is such that $0<\sup_{(Y^*, Y^{*'}),Y}|\rho(Y^*,Y)-\rho(Y^{*'},Y)|\le\Delta_{\rho}<\infty$.
\algoref{ABCDP} is
$\epsilon_{total}$-DP, where
% Define
\begin{align}\label{eq:epsilon_tot_Lap}
\epsilon_{total} = \left\{
    \begin{array}{ll}
        \frac{(c+1)\Delta_\rho}{b} & \mbox{if RESAMPLE is False}, \\
        \frac{2c\Delta_\rho}{b} & \mbox{if RESAMPLE is True}.
    \end{array}
\right.
\end{align}
\end{restatable}

\begin{figure*}[t]
\centering{\includegraphics[width=0.9\textwidth]{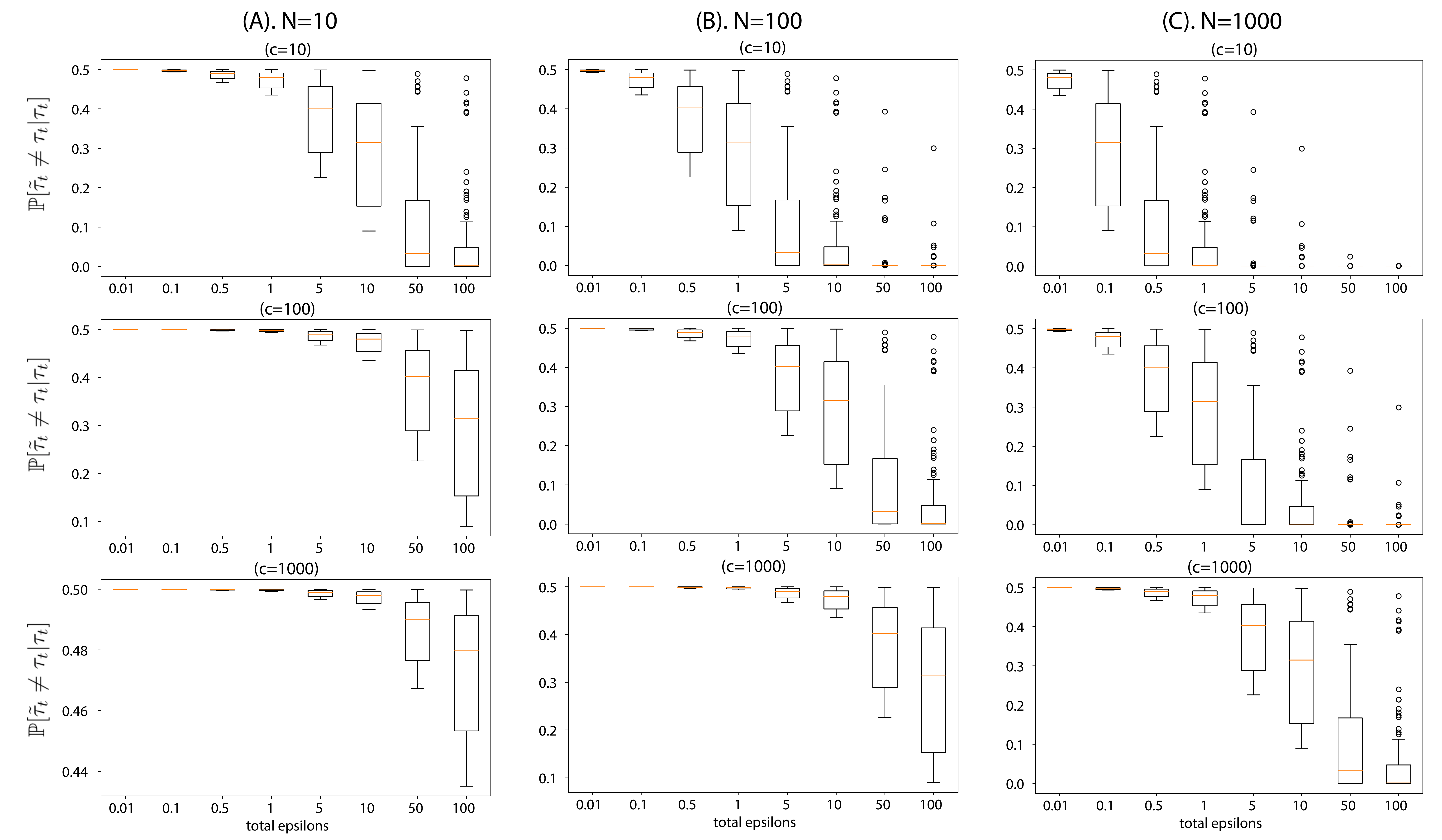}}
\caption{Visualization of flip probability derived in \propref{flip_prob}, the probability that the outputs of ABCDP and ABC differ given an output of ABC, with different dataset size $N$ and accepted posterior sample size $c$. We simulated $\rho \sim \mbox{Uniform}[0,1]$ (drew $100$ values for $\rho$) and used $\epsilon_{abc}=0.2$. \textbf{(A)} This column shows the flip probability at a regime of extremely small datasets, $N=10$. Top plot shows the probability at $c=10$, Middle plot at $c=100$, and Bottom plot at $c=1000$. In this regime, even $\epsilon_{total}=100$ cannot reduce the flip probability to perfectly zero when $c=10$. The flip probability remains high when we accept more samples i.e., $c=1000$. \textbf{(B)} The flip probability at $N=100$. \textbf{(C)} The flip probability at $N=1000$. As we increase the dataset size $N$ (moving from left to right columns), the flip probability approaches zero at a smaller privacy loss $\epsilon_{total}$.
}
\label{fig:flip_prob}
\vspace{-2mm}
\end{figure*}
A proof\footnote{The proof uses linear composition, i.e., the privacy level linearly degrading with $c$. Using the strong composition or more advanced compositions can reduce the resulting privacy loss, while these compositions turn pure-DP to a weaker, approximate-DP. In this paper, we focus on the pure-DP case of ABCDP.} is given in \suppsecref{proof_abcdp_Lap}. For the case of RESAMPLE = True, the proof directly follows the proof of the standard SVT algorithm using the linear composition method \cite{Dwork14}, with an exception that we utilize the quantity representing the minimum noisy value of any query evaluated on $Y^*$, as opposed to the maximum utilized in SVT. For the case of RESAMPLE= False, the proof follows the proof of Algorithm 1 in \cite{SVT_c1}.

Note that the DP analysis in \thmref{abcdp_Lap} holds for other types of distance metrics and not limited to only MMD, as long as there is a bounded sensitivity $\Delta_\rho$ under the chosen metric. When there is no bounded sensitivity, one could impose a clipping bound $C$ to the distance  by taking the distance from $\min[\rho(Y_t,Y^*), C]$, such that the resulting distance between any pseudo data $Y_t$ and $Y^{*'}$ with modifying one datapoint in $Y^*$ cannot exceed that clipping bound.
In fact, we use this trick in our experiments when there is no bounded sensitivity.

\subsection{Effect of noise added to ABC}
\label{sec:noise_flip}

Here, we would like to analyze the effect of noise added to ABC. In particular, we are interested in analyzing the probability that the output of ABCDP differs from that of ABC: $\mathbb{P}[\tilde\tau_t \neq \tau_t|\tau_t]$ at any given time $t$. To compute this probability, we first compute the probability density function of the random variables $m_t-\nu_t$ in the following Lemma.
\begin{lem}
  \label{lemma:pdf_Z}
Recall $m_t \sim \mbox{Lap}(b)$, $\nu_t \sim \mbox{Lap}(2b)$. Subtraction of these yields another random variable $Z = m_t - \nu_t$, where the PDF of $Z$ is given by
\begin{align}
  \label{eq:pdf_Z}
    f_Z(z) &= \frac{1}{6b}\left[2\exp\left(-\frac{|z|}{2b} \right) - \exp\left(-\frac{|z|}{b}\right)
    \right].
\end{align}
Further, for $a \ge 0$, $G_{b}(a):=\int_{a}^{\infty}f_{Z}(z)\thinspace\mathrm{d}z=\frac{1}{6}\left[4\exp\left(-\frac{a}{2b}\right)-\exp\left(-\frac{a}{b}\right)\right]$, and the CDF of Z is given by $F_Z(a) = H[a] + (1-2H[a])G_{b}(|a|)$ where $H[a]$ is the Heaviside step function.
% If Gaussian noise is used, $m_t \sim \Nrm(0, \sigma_1^2)$ and $\nu_t \sim \Nrm(0, 4\sigma_1^2)$. Subtraction of these yields another random variable $Z$, where the pdf is given by
% \begin{align}
%     f_Z(z) = \Phi\left(\frac{-|\rho_t - \epsilon_{abc}|}{\sqrt{5}\sigma_1}\right)
% \end{align}
\end{lem}
See \suppsecref{proof_pdf_z} for the proof.
Using this PDF, we now provide the following proposition:
\begin{prop}\label{prop:flip_prob}
Denote the output of  \algoref{ABCDP} at time $t$ by $\tilde\tau_t  \in \{0,1 \}$ and the output of ABC by $\tau_t  \in \{0,1 \}$. The flip probability, the probability that the outputs of ABCDP and ABC differ given the output of ABC, is given by $\mathbb{P}[\tilde\tau_t \neq \tau_t|\tau_t] = G_{b}(|\rho_t - \epsilon_{abc}|)$,
where $G_{b}(a)$ is defined in Lemma \ref{lemma:pdf_Z}, and $\rho_t := \rho(Y^*, Y_t)$.
\end{prop}
See \suppsecref{proof_flip_prob} for proof.

To provide an intuition of \propref{flip_prob}, we visualize the flip probability  in \figref{flip_prob}.
This flip probability provides a guideline for choosing the accepted sample size $c$ given the datasize $N$ and the desired privacy level $\epsilon_{total}$.
For instance, if a given dataset is extremely small, e.g., containing datapoints on the order of $10$, $c$ has to be chosen such that the flip probability of each posterior sample remains low for a given privacy guarantee ($\epsilon_{total}$). If a higher number of posterior samples are needed, then one has to reduce the desired privacy level for the posterior sample of ABCDP to be similar to that of ABC. Otherwise, with a small $\epsilon_{total}$ with a large $c$, the accepted posterior samples will be poor.
On the other hand, if the dataset is bigger, then a larger $c$ can be taken for a reasonable level of privacy.

%%%%%%%%%%%%%%%%%%%%%%%%%%%%%%%%%%%%%%%%%%%%%%%%%%%%%%%%%%%%%%

\begin{figure*}[t]
\centering{\includegraphics[width=0.9\textwidth]{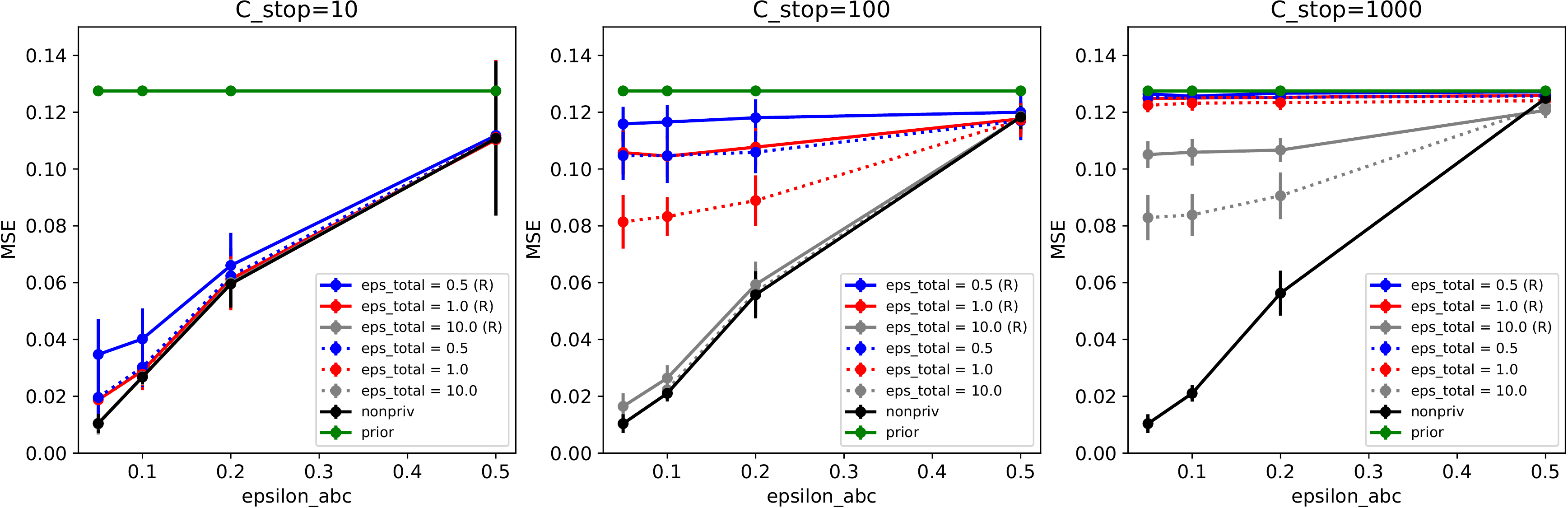}}
\caption{\textbf{ABCDP on Synthetic data}. Mean-squared error (between true parameters and posterior mean) as a function of similarity threshold $\epsilon_{abc}$ given each privacy level. We ran ABCDP with options, \textit{RESAMPLE=True} (denoted by R and solid line) or \textit{RESAMPLE=False} (without R and dotted line) for $60$ independent runs.
\textbf{Left:} When $c_{stop}=10$ at different values of $\epsilon_{abc}$,  ABCDP and non-private ABC (black trace) achieved the highest accuracy (lowest MSE) at the smallest $\epsilon_{abc}$ ($\epsilon_{abc}=0.01$). Notice that ABCDP \textit{RESAMPLE=False} (dotted) outperforms ABCDP \textit{RESAMPLE=True} (solid) for the same privacy tolerance ($\epsilon_{total}$) at small values of $\epsilon_{abc}$.
\textbf{Middle:}  MSE for $c_{stop}=100$ at different values of $\epsilon_{abc}$.
\textbf{Right:}  MSE for $c_{stop}=1000$ at different values of $\epsilon_{abc}$. We can observe when $\epsilon_{abc}$ is large, ABCDP (grey) marginally outperforms non-private ABC (black) due to the excessive noise added in ABCDP. }
\label{fig:toy_rej}
\vspace{-1mm}
\end{figure*}

\subsection{Convergence of Posterior Expectation of Rejection-ABCDP to Rejection-ABC.}
\newcommand{\btheta}{\boldsymbol{\theta}}
The flip probability studied in Section \ref{sec:noise_flip} only accounts
for the effect of noise added to a single output of ABCDP.
Building further on this result, we analyze the discrepancy between the posterior expectations derived from ABCDP and from the rejection ABC.
This analysis requires quantifying the effect of noise added to the whole
sequence of outputs of ABCDP. The result is presented in Thm.
\ref{thm:sparse_posterior_bound}.

%
%In the following analysis, we assume that the noise variable $m_{t}$
%added to the threshold in Algorithm \ref{algo:ABCDP_reject} is redrawn
%in each iteration $t=1,\ldots,T$. This assumption results in the
%independence among the noise realizations and eases theoretical analysis.
%We note that this modification is only for Proposition \ref{prop:sparse_posterior_bound}.

%The first statement in \thmref{sparse_posterior_bound} states that the average error in the posterior expectation computed on a function $f$ under ABCDP boils down to the weighted average of the flip probabilities given in \propref{flip_prob}.  The second statement in \thmref{sparse_posterior_bound} provides the tail behaviour of the error.

\begin{thm}
\label{thm:sparse_posterior_bound}
Given $Y^{*}$ of size $N$,
and $\{(\boldsymbol{\theta}_{t},Y_{t})\}_{t=1}^{T}$ as input, let
$\tilde{\tau}_{t}\in\{0,1\}$ be the output from Algorithm \ref{algo:ABCDP}
where $\tilde{\tau}_{t}=1$ indicates that $(\btheta_{t},Y_{t})$
is accepted, for $t=1,\ldots,T$. Similarly, let $\tau_{t}$ denote
the output from the traditional rejection ABC algorithm, for $t=1,\ldots,T$.
Let $f$ be an arbitrary vector-valued function of $\btheta$. Assume
that the numbers of accepted samples from Algorithm \ref{algo:ABCDP},
and the traditional rejection ABC algorithm are $c:=\sum_{t=1}^{T}\tilde{\tau}_{t}\ge1$
and $c':=\sum_{t=1}^{T}\tau_{t}\ge1$, respectively. Let $b=\frac{4c\sqrt{B}_{k}}{\epsilon_{total}N}$
if RESAMPLE=True, and $b=\frac{2(c+1)\sqrt{B_{k}}}{\epsilon_{total}N}$
if RESAMPLE=False (see Theorem \ref{thm:abcdp_Lap}). Define $K_{T}:=\max_{t=1,\ldots,T}\|f(\boldsymbol{\theta}_{t})\|_{2}$.
Then, the following statements hold for both RESAMPLE options:

1. $\mathbb{E}_{\tilde{\tau}_{1},\ldots,\tilde{\tau}_{T}}\bigg\|\frac{1}{c}\sum_{t=1}^{T}f(\btheta_{t})\tilde{\tau}_{t}-\frac{1}{c'}\sum_{t=1}^{T}f(\btheta_{t})\tau{}_{t}\bigg\|_{2}\le\frac{2K_{T}}{c'}\sum_{t=1}^{T}G_{b}(|\rho_{t}-\epsilon_{abc}|),$
where the decreasing function $G_{b}(x)\in(0,\frac{1}{2}]$ for any
$x\ge0$ is defined in Lemma \ref{lemma:pdf_Z}; \vspace{2mm}

2. $\mathbb{E}_{\tilde{\tau}_{1},\ldots,\tilde{\tau}_{T}}\bigg\|\frac{1}{c}\sum_{t=1}^{T}f(\btheta_{t})\tilde{\tau}_{t}-\frac{1}{c'}\sum_{t=1}^{T}f(\btheta_{t})\tau{}_{t}\bigg\|_{2}\to0$ as $N\to\infty$;
\vspace{2mm}

3. For any $a>0$,
% \begin{align*}
%  & P\bigg(\bigg\|\frac{1}{c}\sum_{t=1}^{T}f(\btheta_{t})\tilde{\tau}_{t}-\frac{1}{c'}\sum_{t=1}^{T}f(\btheta_{t})\tau{}_{t}\bigg\|_{2}\le a\bigg)\ge\\
%  & 1-\frac{4K_{T}}{3ac'}\sum_{t=1}^{T}\exp\left(-\frac{|\rho_{t}-\epsilon_{abc}|}{2b}\right)
% \end{align*}
 $ P\bigg(\bigg\|\frac{1}{c}\sum_{t=1}^{T}f(\btheta_{t})\tilde{\tau}_{t}-\frac{1}{c'}\sum_{t=1}^{T}f(\btheta_{t})\tau{}_{t}\bigg\|_{2}\le a\bigg)\ge
  1-\frac{4K_{T}}{3ac'}\sum_{t=1}^{T}\exp\left(-\frac{|\rho_{t}-\epsilon_{abc}|}{2b}\right)$
where the probability is taken with respect to $\tilde{\tau}_{1},\ldots,\tilde{\tau}_{T}$.
\end{thm}

Thm. \ref{thm:sparse_posterior_bound} contains three statements.
The first states that the expected error between the two posterior expectations of an arbitrary function $f$
is bounded by a constant factor of the sum of the flip
probability in each rejection/acceptance step.
As we have seen in Section \ref{sec:noise_flip}, the flip probability
is determined by the scale parameter $b$ of the Laplace distribution.
Since $b = O(1/N)$ (see Theorem \ref{thm:abcdp_Lap} and Lemma \ref{lem:deltammd}), it follows that
the expected error decays as $N$ increases, giving the second statement.

The third statement gives a probabilistic bound on the error.
The bound guarantees that the error decays exponentially in $N$.
Our proof relies on establishing an upper bound on the error as a
function of the total number of flips $\sum_{t=1}^T |\tilde{\tau}_t - \tau_t|$
which is a random variable. Bounding the error of interest then amounts to characterizing the tail behavior of this quantity.
Observe that in Thm \ref{thm:sparse_posterior_bound} we consider the ABCDP and rejection ABC with the same computational budget i.e.,  the same total number of
iterations $T$ performed. However, the number of accepted samples may be
different in each case ($c$ for ABCDP and $c'$ for reject ABC). The fact that $c$ itself
is a random quantity due to injected noise presents its own technical challenge in
the proof. Our proof can be found in \suppsecref{proof_sparse_posteriior_bound}.

\section{Related Work}
Combining DP with ABC is relatively novel.
The only related work is \cite{gong2019exact}, which states
that a rejection ABC algorithm produces
posterior samples from the exact posterior distribution given perturbed data, when the kernel and bandwidth of rejection ABC are chosen in line with the data perturbation mechanism.
The focus of \cite{gong2019exact} is to identify the condition when the posterior becomes exact in terms of the kernel and bandwidth of the kernel through the lens of data perturbation.
On the other hand, we use the sparse vector technique to reduce the total privacy loss. The resulting theoretical studies including the flip probability and the error bound on the posterior expectation are new.

\section{Experiments}

\begin{figure}[t]
%\centering
%\includegraphics[width=0.25\textwidth]{toy.pdf}
\begin{center}
\subfloat[True parameters]{
    %left bottom right top
\includegraphics[trim={8 85 30 8},clip,width=0.35\textwidth]{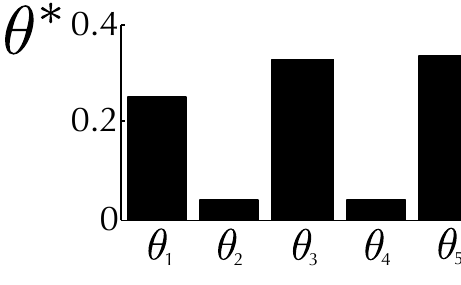}
}
\hspace{14mm}
\subfloat[Observations]{
\includegraphics[trim={8 3 10 90},clip,width=0.35\textwidth]{figs/toy.pdf}
}
\end{center}
\caption{
    Synthetic data. \textbf{(a)}: 5-dimensional true parameters. \textbf{(b)}:
Observations sampled from the mixture of uniform distributions in
(\ref{eq:likelihood_mixture_uniform}) with the true  parameters.
}
\label{fig:toy}
\vspace{-2mm}
\end{figure}

%---------------------------
\subsection{Toy Examples}

%\paragraph{Toy problem}

We start by investigating the interplay between $\epsilon_{abc}$ and $\epsilon_{total}$, in a synthetic dataset where the ground truth parameters are known.
Following \cite{ParJitSej2016}, we also consider a symmetric Dirichlet prior $\pi$ and a likelihood $p(y|\theta)$ given by a mixture of uniform distributions as
\begin{align}\label{eq:likelihood_mixture_uniform}
    \pi(\theta) &= \text{Dirichlet}(\theta; \boldsymbol{1}), \nonumber \\
    p(y|\theta) &= \sum_{i=1}^5 \theta_i \text{Uniform}(y; [i-1, i]).
\end{align}
A vector of mixing proportions is our model parameters $\theta$, where
the ground truth is $\theta^* = [0.25, 0.04, 0.33, 0.04, 0.34]^\top$ (see \figref{toy}).
The goal is to estimate $\mathbb{E}[\theta|Y^*]$ where $Y^*$ is generated with $\theta^*$.

We first generated $5000$ samples for $Y^*$ drawn from
(\ref{eq:likelihood_mixture_uniform}) with true parameters $\theta^*$. Then we tested our two
ABCDP frameworks with varying $\epsilon_{abc}$ and $\epsilon_{total}$.  In
these experiments, we set $\rho =\widehat{\mathrm{MMD}} $ with a Gaussian kernel. We set the
bandwidth of the Gaussian kernel using the median heuristic computed on the
simulated data (i.e., we did not use the real data for this, hence there is no
privacy violation in this regard).
% For soft ABCDP, we drew $5000$ pseudo-samples for $Y_t$ at each time $t$.
% We show the result of our soft ABCDP framework in \figref{toy_soft}.

We drew $5000$ pseudo-samples for $Y_t$ at each time.
We tested various settings, as shown in \figref{toy_rej} where we vary the number of posterior samples, $c=\{10,100,1000\}$,  $\epsilon_{abc}=\{0.05, 0.1, 0.2, 0.5\}$
and $\epsilon_{total}=\{0.5, 1.0, 10, \infty\}$.
We showed the result of ABCDP for both RESAMPLE options in \figref{toy_rej}.
% Note that we do not intend to compare the results of private versions in \figref{toy_soft} and  in \figref{toy_rej}, as this rejection ABCDP provides a pure DP guarantee while soft ABCDP provides an approximate DP with a failure probability 1-$\delta_{total}$ (where we set $\delta_{total}=10^{-4}$).

% It is worth noting two points in \figref{toy_rej}.
% The first point is that while the optimal $\epsilon_{abc}$ is always the smallest for non-private ABC, it is not the case for the private versions of ABCDP. The second point is that in some combinations of $\epsilon_{total}$ and $c$, some private versions of ABCDP surprisingly  outperforms the non-private ABC
% (e.g.,  $\epsilon_{total}=0.1$ and $c=10$ (left in \figref{toy_rej}), and when $\epsilon_{total}=1.0$ and $c=100$ (middle in \figref{toy_rej}), and when $\epsilon_{total}=10$ and $c=1000$ (right in \figref{toy_rej})).
% The reason would be that
% the large $\epsilon_{abc}$ (0.5 in this case) could lead the non-private ABC to accept most of the samples drawn from the prior. However, with noise added to rejection-ABCDP, some samples are rejected rather than mostly being accepted, which can potentially improve the quality of posterior samples. In fact, the large variance across the $15$ independent runs under the non-private ABC (the black error bar) indicates that the performance of the non-private ABC at the large threshold $\epsilon_{abc}=0.5$ is not reliable anymore.
% \mj{explain why this happens}
% Compared to the baseline with samples drawn from the prior, MSE ,...

\begin{figure}[t]
\centering{\includegraphics[width=0.65\textwidth]{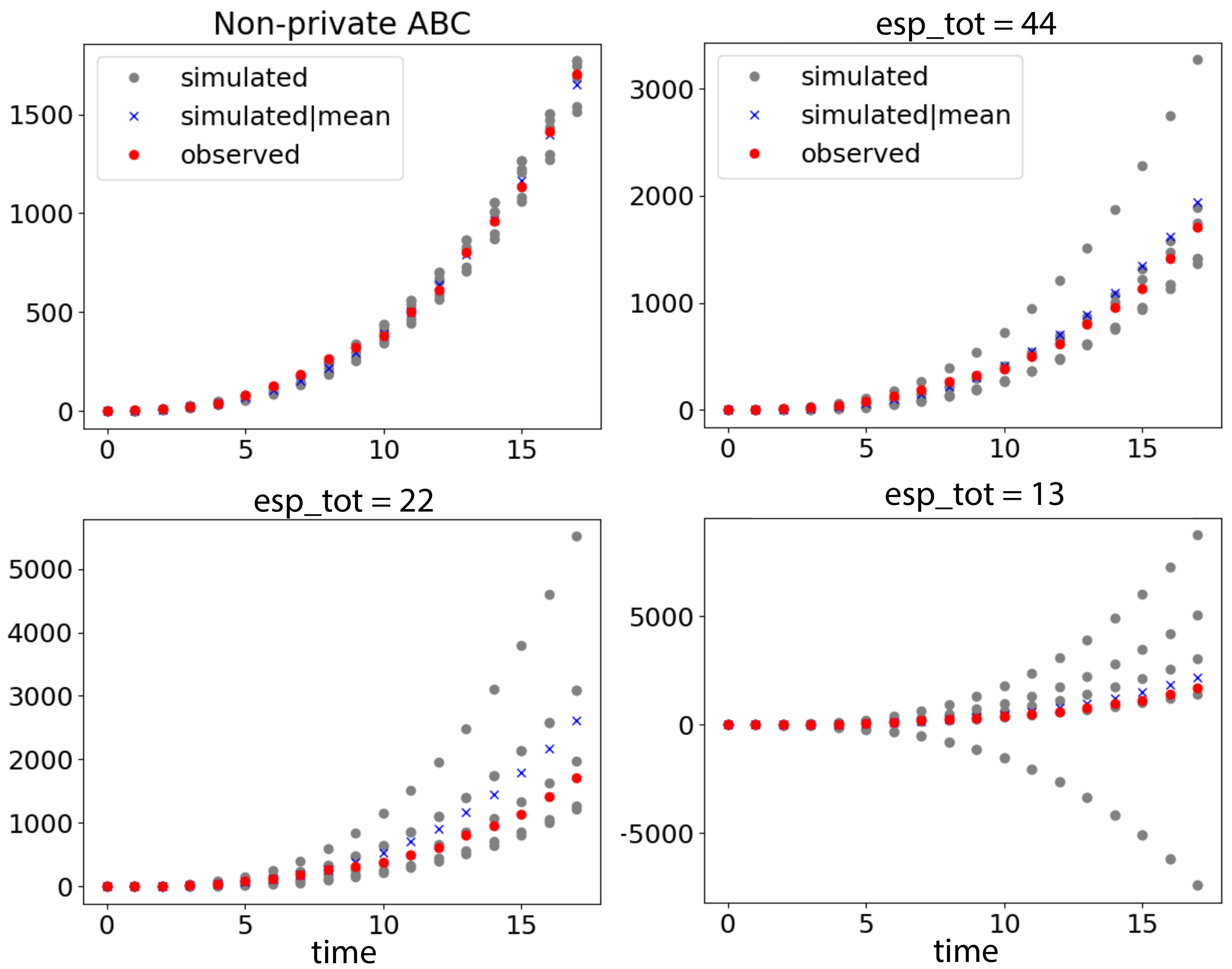}}
\caption{Covid-19 outbreak data ($N=18$) and simulated data  under different privacy guarantees. Red dots show observed data, and grey dots show simulated data drawn from $5$ posterior samples accepted in each case. The blue crosses are simulated data given the posterior mean in each case. \textbf{Top left:} Simulated data by non-private ABC. \textbf{Top right:} Simulated data by ABCDP with $\epsilon_{total}=44$
are relatively well aligned with observed due to the extremely small size of data. \textbf{Bottom left:} The simulated data given $5$ posterior samples exhibit a large variance when $\epsilon_{total}=22$. \textbf{Bottom right:} When $\epsilon_{total}=13$, the simulated data  exhibit an excessively large variance.
}
\label{fig:covid_pred_trj}
\vspace{-2mm}
\end{figure}

    \begin{figure*}[t]
  \begin{center}
    \includegraphics[width=0.75\linewidth]{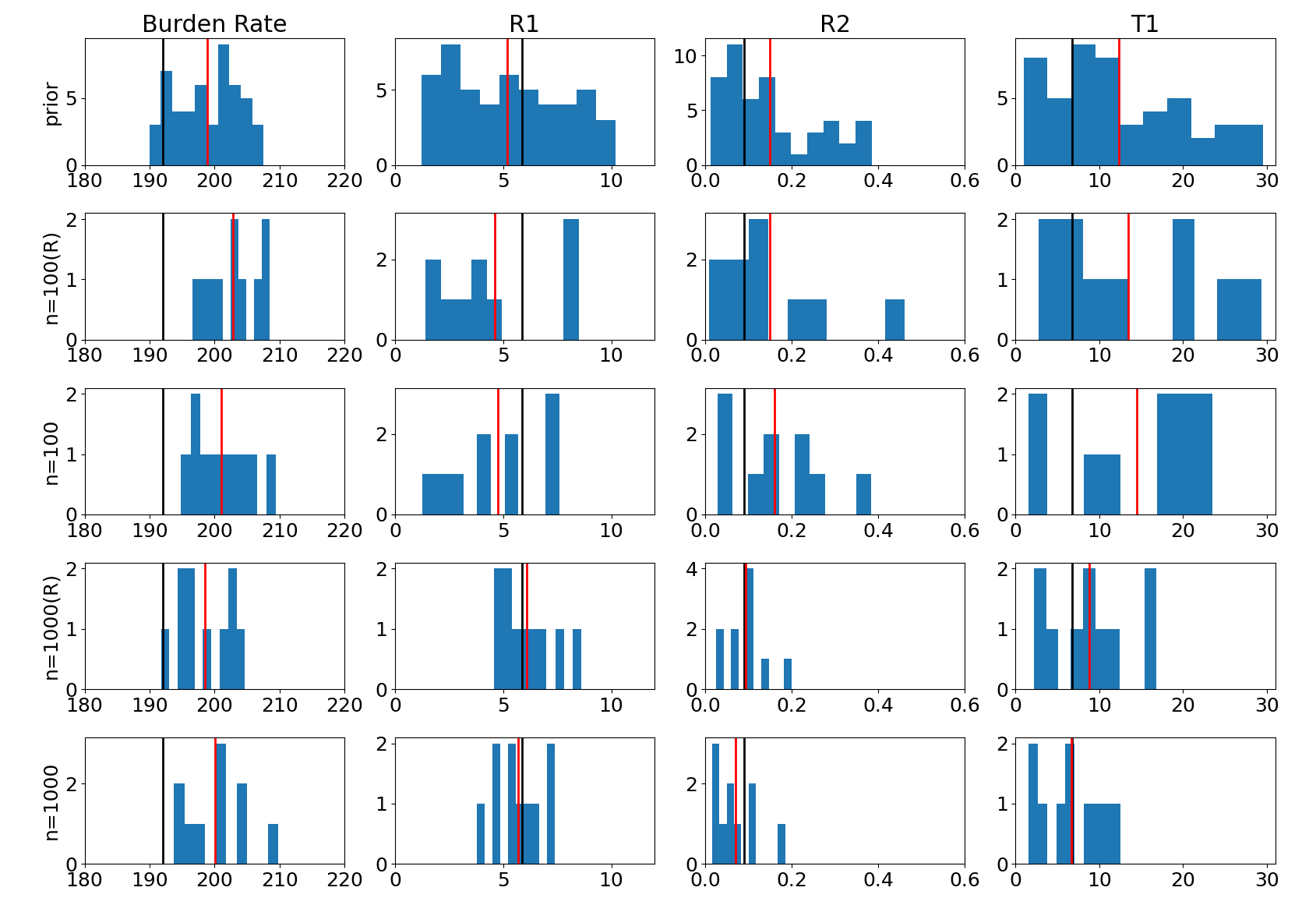}
  \end{center}
  \caption{Posterior samples for Modelling Tuberculosis (TB) Outbreak. In all ABCDP methods, we set $\epsilon_{total}=1$. True values in Black. Mean of samples in Red. (R) indicates ABCDP with Resampling=True. \textbf{1st row}: Histogram of $50$ samples drawn from the prior (we used the same prior as \cite{TB2019}).  \textbf{2nd row}: $10$ posterior samples from ABCDP with (R) given $n=100$ observations.
  \textbf{3rd row}: $10$ posterior samples from ABCDP without (R) given $n=100$ observations.
%   By looking at the second and third rows, ABCDP without resampling of the noise for threshold (third row) performs slightly better than the second row, based on the distance between the black bar (true) and red bar (estimate).
  \textbf{4th row}: $10$ posterior samples from ABCDP with (R) given $n=1000$ observations.
  \textbf{5th row}: $10$ posterior samples from ABCDP without (R) given $n=1000$ observations.
  The distance between the black bar (true) and red bar (estimate) reduces as the size of data increases from $100$ to $1000$. ABCDP with Resampling=False performs better regardless of the data size.
  }
  \label{fig:TB}
  \end{figure*}

%-------------------------------------------
\subsection{Coronavirus Outbreak Data}
% make inference on the future growth of the infected cases
In this experiment, we model  coronavirus outbreak in the Netherlands using
a polynomial model consisting of four parameters $a_0, a_1, a_2, a_3$, which we aim to infer, where
\begin{align}
y(t)= a_3 + a_2 t + a_1 t^2 + a_0 t^3.
\end{align}
The observed\footnote{{https://www.ecdc.europa.eu/en/publications-data/download-todays-data-geographic-distribution-covid-19-cases-worldwide}.
} data are the number of cases of the coronavirus outbreak from Feb 27 to March 17, 2020, which amounts to 18 datapoints ($N=18$). The presented experiment imposes privacy concern as each datapoint is a count of the individuals who are COVID positive at each time.
The goal is to identify the
approximate posterior distribution $\tilde{p}(a_0, a_1, a_2, a_3|y^*)$ over these parameters given a set of
observations.

Recalling from \figref{flip_prob} that the small size of data worsens the privacy and accuracy trade-off,  the inference is restricted to a small number of posterior samples (we chose $c=5$) since the number of datapoints is extremely limited in this dataset .
We used the same prior distributions for the four parameters as $a_i \sim \Nrm(0, 1)$ for all $i=0,1,2,3$. We drew $50,000$ samples from the Gaussian prior, and performed our ABCDP algorithm with $\epsilon_{total}=\{13, 22, 44\}$ and $\epsilon_{abc}=0.1$, as shown in \figref{covid_pred_trj}.

%-------------------------------------------
%-------------------------------------------
\subsection{Modelling Tuberculosis (TB) Outbreak Using Stochastic Birth-Death Models}

In this experiment, we used the  stochastic birth-death models to model  Tuberculosis (TB) outbreak.
% The observed data are the tuberculosis epidemic data from San Francisco Bay area \cite{doi:10.1056/NEJM199406163302402}.
%
There are four parameters that we aim to infer, which go into the Communicable
disease outbreak simulator as inputs: burden rate $\beta$, transmission rate
$t_1$, reproductive numbers $R_1$ and $R_2$. The goal is to identify the
approximate posterior distribution $\tilde{p}(R_1, t_1, R_2, \beta|y^*)$ over these parameters given a set of
observations. Please refer to Sec 3 in \cite{TB2019} for the description of the birth-death process of the model.
We used the same prior distributions for the four parameters as in \cite{TB2019}: $\beta \sim \Nrm(200, 30)$, $R_1\sim\mbox{Unif}(1.01, 20)$, $R_2\sim\mbox{Unif}(0.01, (1-0.05R_1)/0.95)$, $t_1 \sim \mbox{Unif}(0.01, 30)$.

%   \vspace{-4mm}
% \end{wrapfigure}
%
To illustrate the privacy and accuracy trade-off, we first generated two sets
of observations $y^*$ ($n=100$ and $n=1000$)
by some \textit{true} model
parameters (shown as black bars in \figref{TB}). We then tested our ABCDP algorithm
with a privacy level $\epsilon=1$. We used the summary statistic described in
Table 1 in \cite{TB2019} and used weighted L2 distance as $\rho$ as done in \cite{TB2019}, together with $\epsilon_{abc}=150$.
Since there is no bounded sensitivity in this case, we impose an artificial boundedness by clipping the distance by $C$ (we set $C=200$) when the distance goes beyond $C$.

As an error metric, we computed the mean absolute distance between each
posterior mean and the true parameter values. The top row in \figref{TB}
shows that the mean of the prior (red) is far from the true value (black) we
chose. As we increase the data size from $n=100$ (middle) to $n=1000$
(bottom), the distance between true values and estimates reduces, reflected in
the error from $4.71$ to $2.20$ for RESAMPLE=True, and from $4.51$ to $2.10$ for RESAMPLE=False.

%-------------------------------------------

\section{Summary and Discussion}

%\mj{let's summarize and talk about future work}
We presented ABCDP algorithm by combining DP with ABC. Our method outputs differentially private binary indicators, yielding differentially-private posterior samples. We derived the error in ABCDP in terms of indicator's flip probability, and also the average error bound of the posterior expectation.
A natural by-product of ABCDP is differentially private synthetic data, as the simulator is a public tool that anybody can run and hence differentially private posterior samples suffice differentially private synthetic data without any further privacy cost. Applying ABCDP to generating complex datasets is an intriguing future direction.

%\wjsay{Redundant?}
%Furthermore, in this current work, we derived the $\epsilon$-DP guarantee for
%rejection ABCDP.
% As a future direction, investigating an approximate DP version of rejection
% ABCDP would be beneficial to make the current framework more practical.

\section*{Acknowledgement}

M. Park and M. Vinaroz are supported by the Max Planck Society and the Gips Schule Foundation and the Institutional ¨
Strategy of the University of T\"ubingen (ZUK63). M. Vinaroz
thanks the International Max Planck Research School for
Intelligent Systems (IMPRS-IS) for its support.

\newpage

\bibliographystyle{plain}
\bibliography{abcdp}

% ------------ appendix ---------------

\onecolumn
\appendix

\begin{center}
{\LARGE{}{}{}{}{}{}\ourtitle{}}
\par\end{center}

\begin{center}
\textcolor{black}{\Large{}{}{}{}{}{}Supplementary}{\Large{}{}{}{}{}
}
\par\end{center}

\setcounter{section}{0}

\section{PROOF OF THEOREM \thmref{abcdp_Lap}}
\label{supp:proof_abcdp_Lap}
We first recall Theorem \thmref{abcdp_Lap} stated in the main text:
\mrejdp*
\begin{proof}

\textbf{Case I: $RESAMPLE= True$.} We prove the case of $c=1$ first. The case of $c > 1$ is a $c$-composition of the case of $c=1$, where the privacy loss linearly increases with $c$.

Given any neighboring datasets $Y^*$ and $Y^{*'}$ of size $N$ and any dataset $Y$, assume that $\rho$ is such that $0<\sup_{(Y^*,Y^{*'}),Y} \mid \rho (Y^*,Y) - \rho(Y^{*'},Y) \mid < \Delta_{\rho} < \infty$ and $\rho$ it's bounded by $B_{\rho}$.
\\
\\
Let $A$ denote the random variable that represents the outputs \algoref{ABCDP} given ($\{(\bm \theta_t, Y_t )\}_{t=1}^T$, $Y^{*}$, $\rho$, $\epsilon_{abc}$, $\epsilon$) and $A'$ the random variable that represents the outputs given ($\{(\bm \theta_t, Y_t )\}_{t=1}^T$, $Y^{*'}$, $\rho$, $\epsilon_{abc}$, $\epsilon$).
The output of the algorithm is some realization of these variables, $\tau \in \{1,0 \}^{k}$ where $0<k \leq T$ and for all $t<k$, $\tau_t = 0$ and $\tau_{k}= 1$.
% \wjsay{Probably easiest to use \{0,1\} throughout instead of $\{\top, \bot\}$.}
For the rest of the proof we will fix the arbitrary values of $\nu_{1},...,\nu_{k-1}$ and take probabilities over the randomness of $\nu_{k}$ and $\epsilon_{abc}$. We define the deterministic quantity ($\nu_{1},...,\nu_{k-1}$ are fixed) :

\begin{equation}
  g(Y^*) = \min_{t<k} (\rho(Y_{t}, Y^{*}) + \nu_t)
\end{equation}

\vspace{5pt}
that represents the minimum noised value of the distance evaluated on any dataset $Y^*$.

Let $Pr[\hat{\epsilon}_{abc} = a]$ be the pdf of $\hat{\epsilon}_{abc}$ evaluated on $a$ and $Pr[\nu_{k}= v]$ the pdf of $\nu_{k}$ evaluated on $v$, and $\mathds{1}[x]$ the indicator function of event $x$. We have:

$$\text{Pr}_{\hat{e}_{abc}, \nu_{k}}[A=\tau_{k}] = \text{Pr}[\hat{\epsilon}_{abc} < g(Y^*)\text{ and } \rho(Y_{k}, Y^{*}) + \nu_k \leq \hat{\epsilon}_{abc}]= \text{Pr}[\hat{\epsilon}_{abc} \in [ \rho(Y_{k}, Y^{*}) + \nu_k, g(Y^*) )]$$

$$=\int_{- \infty}^{\infty}\int_{- \infty}^{\infty} \text{Pr}[\nu_{k}=v] \text{Pr}[\hat{\epsilon}_{abc} = a]\mathds{1}[a \in [\rho(Y_{k}, Y^{*}) + \nu_{k}, g(Y^*))] dv da$$

Now, we define the following variables:

$$\hat{a} = a + g(Y^*) - g(Y^{*'})$$

$$\hat{v} = v +  g(Y^*) - g(Y^{*'}) + \rho(Y_{k}, Y^{*'})- \rho(Y_{k}, Y^{*}) $$

We know that for each $Y^*, Y^{*'}$, $\rho$ is $\Delta_{\rho}$-sensitive and hence, the quantity $g(Y^*)$ is $\Delta_{\rho}$-sensitive as well. In this way we obtain that $ \mid \hat{a} - a \mid \leq \Delta_{\rho}$ and $\mid \hat{v} - v \mid \leq 2\Delta_{\rho}$. Applying these changes of variables we have:

$$ = \int_{- \infty}^{\infty}\int_{- \infty}^{\infty} \text{Pr}[\nu_{k}=\hat{v}] \text{Pr}[\hat{\epsilon}_{abc} = \hat{a}]\mathds{1}[a + g(Y^*) - g(Y^{*'}) \in [v + g(Y^*) - g(Y^{*'}) + \rho(Y_{k}, Y^{*'}), g(Y^*) )] dv da$$

$$= \int_{- \infty}^{\infty}\int_{- \infty}^{\infty} \text{Pr}[\nu_{k}=\hat{v}]\text{Pr}[\hat{\epsilon}_{abc} = \hat{a}] \mathds{1}[a \in [v + \rho(Y_{k}, Y^{*'}), g(Y^{*'}) )] dv da $$

$$\leq \int_{- \infty}^{\infty}\int_{- \infty}^{\infty} \exp{\Big(\frac{\epsilon}{2}\Big)} \text{Pr}[\nu_k = v] \exp{\Big(\frac{\epsilon}{2}\Big)} \text{Pr}[\hat{\epsilon}_{abc} = a]\mathds{1}[a \in [v + \rho(Y_{k}, Y^{*'}), g(Y^{*'}) )] dv da$$

$$\leq \exp{( \epsilon )} \int_{- \infty}^{\infty}\int_{- \infty}^{\infty} \text{Pr}[\nu_k = v] \text{Pr}[\hat{\epsilon}_{abc} = a]\mathds{1}[a \in [v + \rho(Y_{k}, Y^{*'})), g(Y^{*'}) )] dv da$$

$$= \exp{(\epsilon)} \text{Pr}[\hat{\epsilon}_{abc} < g(Y^{*'})\text{ and } \rho(Y_{k}, Y^{*'}) + \nu_k \leq \hat{\epsilon}_{abc}] = \exp{(\epsilon)} \text{Pr}_{\hat{e}_{abc}, \nu_{k}}[A^{'}=\tau_{k}]$$

where the inequality comes from the bounds considered throughout the proof  (i.e. $\mid \hat{a} - a \mid \leq \Delta_{\rho}$ and $\mid \hat{v} - v \mid \leq 2\Delta_{\rho}$)  and the form of the cdf for the Laplace distribution.

\textbf{Case II: $RESAMPLE= False$.}
In this case, the proof follows the proof of Algorithm 1 in \cite{SVT_c1}, with an exception that positive events for \cite{SVT_c1}
becomes negative events for us and vice versa as we find the value below a threshold, where \cite{SVT_c1} finds the value above a threshold.
$$$$
\end{proof}

\section{PROOF OF LEMMA \ref{lem:deltammd}}
\label{supp:proof_deltammd}
Recall Lemma \ref{lem:deltammd} stated in the main text:
\deltammd*

\begin{proof}
We will establish $\Delta_{\rho}$ when $\rho$ is MMD. Recall that
$(Y^*,Y^{*'})$ is a pair of neighbouring datasets, and $Y$ is an arbitrary
dataset. Without loss of generality, assume that $Y^*=\{x_{1},\ldots,x_{N}\}$,
$Y^{*'}=\{x_{1}',\ldots,x_{N}'\}$ such that $x_{i}=x_{i}'$ for all $i=1,\ldots,N-1$,
and $Y=\{y_{1},\ldots,y_{m}\}$.
 We start with
\begin{align*}
 & \sup_{(Y^*,Y^{*'}),Y}|\rho(Y^{*},Y)-\rho(Y^{*'},Y)|\\
 & =\sup_{(Y^*,Y^{*'}),Y}|\widehat{\mathrm{MMD}}(Y^{*},Y)-\widehat{\mathrm{MMD}}(Y^{*'},Y)|\\
 & =\sup_{(Y^*,Y^{*'}),Y}\bigg|\thinspace\bigg\|\frac{1}{N}\sum_{i=1}^{N}\phi(x_{i})-\frac{1}{m}\sum_{j=1}^{m}\phi(y_{j})\bigg\|_{\mathcal{H}}
 -\bigg\|\frac{1}{N}\sum_{i=1}^{N}\phi(x_{i}')-\frac{1}{m}\sum_{j=1}^{m}\phi(y_{j})\bigg\|_{\mathcal{H}}\thinspace\bigg|\\
 & \stackrel{(a)}{\le}\sup_{(X,X')}\bigg\|\frac{1}{N}\sum_{i=1}^{N}\phi(x_{i})-\frac{1}{N}\sum_{i=1}^{N}\phi(x_{i}')\bigg\|_{\mathcal{H}}\\
 & =\sup_{(x_{N},x_{N}')}\bigg\|\frac{1}{N}\phi(x_{N})-\frac{1}{N}\phi(x_{N}')\bigg\|_{\mathcal{H}}\\
 & =\sup_{(x_{N},x_{N}')}\frac{1}{N}\sqrt{k(x_{N},x_{N})+k(x_{N}',x_{N}')-2k(x_{N},x_{N}')}\\
 & \le\frac{2}{N}\sqrt{B_{k}},
\end{align*}
where at $(a)$ we use the reverse triangle inequality.
Furthermore,
\begin{align*}
 & \sup_{Y^*,Y}\rho(Y^*,Y)\\
 & \le\sup_{Y^*,Y}\sqrt{\bigg\|\frac{1}{N}\sum_{i=1}^{N}\phi(x_{i})-\frac{1}{m}\sum_{i=1}^{m}\phi(y_{i})\bigg\|_{\mathcal{H}}^{2}}\\
 & =\sup_{Y^*,Y}\sqrt{\frac{1}{N^{2}}\sum_{i,j=1}^{N}k(x_{i},x_{j})+\frac{1}{m^{2}}\sum_{i,j=1}^{m}k(y_{i},y_{j})-\frac{2}{mn}\sum_{i=1}^{N}\sum_{j=1}^{m}k(x_{i},y_{j})}\\
 & =\sqrt{B_{k}+B_{k}+2B_{k}}=2\sqrt{B_{k}}.
\end{align*}

\end{proof}

\section{PROOF OF \lemref{pdf_Z}}\label{supp:proof_pdf_z}

The PDF is computed from the convolution of two PDFs.
% \url{https://stats.stackexchange.com/questions/100746/difference-between-two-i-i-d-laplace-distributions}.
%
\begin{align}
    f_{m_t-\nu_t}(z) = \int_{-\infty}^\infty f_{m_t}(x) f_{\nu_t}(x-z) dx,
\end{align} where $f_{m_t}(x)=\frac{1}{2b}\exp(-\frac{|x|}{b})$ and $f_{\nu_t}(y) = \frac{1}{4b}\exp(-\frac{|y|}{2b})$.
\begin{align}
    f_{m_t-\nu_t}(z) &=
    \frac{1}{8b^2}\int_{-\infty}^\infty
    \exp\left( -\frac{|x|}{b} - \frac{|x-z|}{2b} \right) dx
\end{align}

\textbf{For case $z\geq0$:}
\begin{align}
    f_{m_t-\nu_t}(z) &=
    \frac{1}{8b^2}\int_{-\infty}^0
    \exp\left( \frac{x}{b}+ \frac{x-z}{2b} \right) dx
    + \frac{1}{8b^2}\int_{0}^z
    \exp\left( -\frac{x}{b} + \frac{x-z}{2b} \right) dx \nonumber\\
    & \quad + \frac{1}{8b^2}\int_{z}^\infty
    \exp\left( -\frac{x}{b} - \frac{x-z}{2b} \right) dx, \\
    &= \frac{1}{8b^2}\int_{-\infty}^0
    \exp\left(\frac{3x-z}{2b} \right) dx
        + \frac{1}{8b^2} \int_{0}^z
    \exp\left(\frac{-x-z}{2b} \right) dx \nonumber \\
    & \quad + \frac{1}{8b^2}\int_{z}^\infty
    \exp\left(\frac{-3x+z}{2b} \right) dx \\
    &= \frac{\exp\left(\frac{-z}{2b} \right)}{8b^2}\int_{-\infty}^0
    \exp\left(\frac{3x}{2b} \right) dx
    + \frac{ \exp\left(\frac{-z}{2b} \right)}{8b^2} \int_{0}^z
    \exp\left(\frac{-x}{2b} \right) dx \nonumber \\
     & \quad + \frac{ \exp\left(\frac{z}{2b} \right)}{8b^2}\int_{z}^\infty
    \exp\left(\frac{-3x}{2b} \right) dx \\
    &= \frac{\exp\left(\frac{-z}{2b} \right)}{8b^2}\frac{2b}{3}
    -\frac{ \exp\left(\frac{-z}{2b} \right)}{8b^2} 2b \left(\exp\left(\frac{-z}{2b} \right) - 1 \right)
    + \frac{ \exp\left(\frac{z}{2b} \right)}{8b^2}\frac{2b}{3} \exp\left(\frac{-3z}{2b} \right), \\
    &=\frac{1}{12b}\left[\exp\left(\frac{-z}{2b} \right) + 3\exp\left(\frac{-z}{2b}\right) \left(1- \exp\left(\frac{-z}{2b} \right)\right)
    + \exp\left(\frac{-z}{b} \right)
    \right], \\
    &=\frac{1}{12b}\left[ 4\exp\left(\frac{-z}{2b}\right)
    -2\exp\left(\frac{-z}{b} \right)
    \right], \\
    &= \frac{1}{6b}\left[ 2\exp\left(\frac{-z}{2b}\right)
    -\exp\left(\frac{-z}{b} \right)
    \right]
\end{align}

% \wjsay{It might be simpler for the reader to just say that the density of $Z$ is symmetric around 0. Therefore, $f_{m_t - \nu_t}(z) = f_{m_t - \nu_t}(-z)$. We do not need to show the derivation for $z<0$.}
\textbf{For case $z<0$:}
\begin{align}
    f_{m_t-\nu_t}(z) &=
    \frac{1}{8b^2}\int_{-\infty}^z
    \exp\left( \frac{x}{b}+ \frac{x-z}{2b} \right) dx
    + \frac{1}{8b^2}\int_{z}^0
    \exp\left( \frac{x}{b} - \frac{x-z}{2b} \right) dx \nonumber\\
    & \quad + \frac{1}{8b^2}\int_{0}^\infty
    \exp\left( -\frac{x}{b} - \frac{x-z}{2b} \right) dx, \\
    &= \frac{1}{8b^2}\int_{-\infty}^z
    \exp\left(\frac{3x-z}{2b} \right) dx
        + \frac{1}{8b^2} \int_{z}^0
    \exp\left(\frac{x+z}{2b} \right) dx \nonumber \\
    & \quad + \frac{1}{8b^2}\int_{0}^\infty
    \exp\left(\frac{-3x+z}{2b} \right) dx \\
    &= \frac{\exp\left(\frac{-z}{2b} \right)}{8b^2}\int_{-\infty}^z
    \exp\left(\frac{3x}{2b} \right) dx
    + \frac{ \exp\left(\frac{z}{2b} \right)}{8b^2} \int_{z}^0
    \exp\left(\frac{x}{2b} \right) dx \nonumber \\
     & \quad + \frac{ \exp\left(\frac{z}{2b} \right)}{8b^2}\int_{0}^\infty
    \exp\left(\frac{-3x}{2b} \right) dx \\
    &= \frac{\exp\left(\frac{-z}{2b} \right)}{8b^2}\frac{2b}{3} \exp\left(\frac{3z}{2b} \right)
    +\frac{ \exp\left(\frac{z}{2b} \right)}{8b^2} 2b \left(1- \exp\left(\frac{z}{2b} \right) \right)
    + \frac{ \exp\left(\frac{z}{2b} \right)}{8b^2}\frac{2b}{3}, \\
    &=\frac{1}{12b}\left[\exp\left(\frac{z}{b} \right) - 3\exp\left(\frac{z}{2b}\right) \left(\exp\left(\frac{z}{2b} \right)-1\right)
    + \exp\left(\frac{z}{2b} \right)
    \right], \\
    &=\frac{1}{12b}\left[-2\exp\left(\frac{z}{b} \right) +4 \exp\left(\frac{z}{2b}\right)
    \right], \\
    &=\frac{1}{6b}\left[2\exp\left(\frac{z}{2b} \right) - \exp\left(\frac{z}{b}\right)
    \right].
\end{align}

With the obtained PDF $f_Z(z) = \frac{1}{6b}\left[2\exp\left(-\frac{|z|}{2b}
\right) - \exp\left(-\frac{|z|}{b}\right) \right]. $ for $Z:=m_t - \nu_t$, it
is straightforward to compute
$G_{b}(a):=\int_{a}^{\infty}f_{Z}(z)\thinspace\mathrm{d}z=\frac{1}{6}\left[4\exp\left(-\frac{a}{2b}\right)-\exp\left(-\frac{a}{b}\right)\right]$ for $a \ge 0$.
In other words, $G_b(a) = 1-F_Z(a)$ for $a \ge 0$ where $F_Z$ denotes the CDF of $Z$.

To show that the CDF of $Z$ is $F_Z(a) = H[a] + (1-2H[a])G_b(|a|)$ where
$H[a]$ is the Heaviside step function, we note that the density $f_Z(z)$ is
an even function i.e., $f_Z(z) = f_Z(-z)$ for any $z$. It follows that
if $a < 0$, $1-F_z(a) = 1-G_b(-a)$. That is,

\begin{equation*}
1-F_{Z}(a)	=\begin{cases}
G_{b}(a) & \text{if }a\ge0,\\
1-G_{b}(-a) & \text{if }a<0,
\end{cases}
\end{equation*}
or equivalently
\begin{equation*}
F_{Z}(a)	=\begin{cases}
1-G_{b}(a) & \text{if }a\ge0,\\
G_{b}(-a) & \text{if }a<0.
\end{cases}
\end{equation*}
More concisely,
\begin{align*}
F_{Z}(a)	&=(1-G_{b}(|a|))\mathbb{I}[a\ge0]+G_{b}(|a|)\mathbb{I}[a<0] \\
	&=\mathbb{I}[a\ge0]+\left(\mathbb{I}[a<0]-\mathbb{I}[a\ge0]\right)G_{b}(|a|)  \\
  &\stackrel{(a)}{=}H[a]+(1-2H[a])G_{b}(|a|),
\end{align*}
where at (a) we use
  $\left(\mathbb{I}[a<0]-\mathbb{I}[a\ge0]\right)=(1-2H[a])$.

% \begin{align*}
% \int_a^\infty f_Z(z) \, \mathrm{d}z
% &= \int_a^0f_Z(z) \, \mathrm{d}z + \int_0^\infty f_Z(z) \, \mathrm{d}z \\
% &= \int_a^0f_Z(-z) \, \mathrm{d}z + G_b(0) \\
% = 1-G_b(-a).
% \end{align*}

\section{PROOF OF \propref{flip_prob}}\label{supp:proof_flip_prob}

Using this pdf above, we can compute the probabilities:
\begin{align}
    &\mbox{Pr}[\tilde\tau_t=1 | \tau_t=0]\\
    &= \mbox{Pr}[0 \le \rho_t -\epsilon_{abc}  \leq  Z],  \\
    &= \int_{\rho_t -\epsilon_{abc}}^\infty f(z) dz, \quad \mbox{where }  \rho_t - \epsilon_{abc} \ge 0 \\
    &=\int_{\rho_t -\epsilon_{abc}}^\infty \frac{1}{6b}\left[2\exp\left(-\frac{|z|}{2b} \right) - \exp\left(-\frac{|z|}{b}\right) \right] dz, \mbox{ by definition of $f(z)$} \\
    &=\int_{\rho_t -\epsilon_{abc}}^\infty \frac{1}{6b}\left[2\exp\left(-\frac{z}{2b} \right) - \exp\left(-\frac{z}{b}\right)
    \right] dz, \quad \mbox{ because }  \rho_t - \epsilon_{abc} \ge 0  \\
    &= \frac{1}{6b}\left[ 4b \exp\left(-\frac{\rho_t-\epsilon_{abc}}{2b} \right) - b \exp\left(-\frac{\rho_t-\epsilon_{abc}}{b} \right) \right],  \\
    &=\frac{1}{6}\left[4\exp\left(-\frac{\rho_t-\epsilon_{abc}}{2b} \right) - \exp\left(-\frac{\rho_t-\epsilon_{abc}}{b} \right) \right], \mbox{where } \rho_t - \epsilon_{abc} \ge 0 ,
    \label{eq:cdf_abcdp_bernoulli}
\end{align} and
% \wjsay{I got the following expression for $\mbox{Pr}[\tilde\tau_t=1 | \tau_t=0]$:
% $\frac{1}{6}\left[4\exp\left(-\frac{\rho_t-\epsilon_{abc}}{2 b} \right) - \exp\left(-\frac{\rho_t-\epsilon_{abc}}{b} \right) \right]$. Please check.
% }\mj{checked and corrected.}
\begin{align}
    &\mbox{Pr}[\tilde\tau_t=0 | \tau_t = 1] \nonumber \\
    &= \mbox{Pr}[Z \leq \rho_t -\epsilon_{abc}  \leq  0],  \\
    &= \int_{-\infty}^{\rho_t -\epsilon_{abc}} f(z) dz, \mbox{ where }  \rho_t -\epsilon_{abc} \leq 0,  \\
    &=\int_{-\infty}^{\rho_t -\epsilon_{abc}}\frac{1}{6b}\left[2\exp\left(\frac{z}{2b} \right) - \exp\left(\frac{z}{b}\right)
    \right] dz, \\
    &= \frac{1}{6b}\left[4b\exp\left( \frac{\rho_t -\epsilon_{abc}}{2b} \right) - b \exp\left( \frac{\rho_t -\epsilon_{abc}}{b} \right)\right], \\
    &= \frac{1}{6}\left[4\exp\left( \frac{\rho_t -\epsilon_{abc}}{2b} \right) -  \exp\left( \frac{\rho_t -\epsilon_{abc}}{b} \right)\right], \mbox{ where } \rho_t - \epsilon_{abc} \le 0.
\end{align}

So,
\begin{align}
    &\mathbb{P}[\tilde\tau_t \neq \tau_t|\tau_t] = \begin{cases}
     \mathbb{P}[\tilde\tau_t=1|\tau_t=0], & \text{if } \rho_t \ge \epsilon_{abc} \\
    \mathbb{P}[\tilde\tau_t=0|\tau_t=1], & \text{otherwise}
\end{cases} \\
&= \begin{cases} \frac{1}{6}\left[4\exp\left(-\frac{\rho_t-\epsilon_{abc}}{2b} \right) - \exp\left(-\frac{\rho_t-\epsilon_{abc}}{b} \right) \right], \text{if } \rho_t\geq \epsilon_{abc} \nonumber \\
    \frac{1}{6}\left[4\exp\left(\frac{\rho_t-\epsilon_{abc}}{2b} \right) - \exp\left(\frac{\rho_t-\epsilon_{abc}}{b} \right) \right], \text{otherwise}. \nonumber
    \end{cases}
    \end{align}
The two cases can be combined with the use of an absolute value to give the
result.

\section{PROOF OF THEOREM \ref{thm:sparse_posterior_bound}}
\label{supp:proof_sparse_posteriior_bound}
\begin{proof}
Let $H(x)$ be the Heaviside step function. Recall from our algorithm
that each accepted sample $(\btheta,Y)$ is associated with two independent
noise realizations: $m_{t}\sim\mathrm{Lap}(b)$ (i.e., $\hat{\epsilon}_{abc}=\epsilon_{abc}+m_{t}$)
and $\nu_{t}\sim\mathrm{Lap}(2b)$ (added to $\rho(Y^{*},Y_{t}$)).
With this notation, we have $\tilde{\tau}_{t}=H[\epsilon_{abc}-\rho(Y_{t},Y^{*})+m_{t}-\nu_{t}]$
for $t=1,\ldots,T$. Similarly, $\tau{}_{t}:=H[\epsilon_{abc}-\rho(Y_{t},Y^{*})]$.
For brevity, we define $\rho_{t}:=\rho(Y_{t},Y^{*})$. It follows
that $\tilde{\tau}_{t}\sim\mathrm{Bernoulli}(p_{t})$ where $p_{t}:=\pr(m_{t}-\nu_{t}>\rho_{t}-\epsilon_{abc})=\pr(\tilde{\tau}=1)$.

\textbf{Proof of the first claim} We start by establishing an upper
bound for
\begin{align}
 & \bigg\|\frac{1}{c}\sum_{t=1}^{T}f(\btheta_{t})\tilde{\tau}_{t}-\frac{1}{c'}\sum_{t=1}^{T}f(\btheta_{t})\tau{}_{t}\bigg\|_{2}\nonumber \\
 & =\bigg\|\frac{1}{c}\sum_{t=1}^{T}f(\btheta_{t})\tilde{\tau}_{t}{\color{blue}-\frac{1}{c'}\sum_{t=1}^{T}f(\btheta_{t})\tilde{\tau}_{t}+\frac{1}{c'}\sum_{t=1}^{T}f(\btheta_{t})\tilde{\tau}_{t}}-\frac{1}{c'}\sum_{t=1}^{T}f(\btheta_{t})\tau{}_{t}\bigg\|_{2}\nonumber \\
 & \le\left|\frac{1}{c}-\frac{1}{c'}\right|\bigg\|\sum_{t=1}^{T}f(\btheta_{t})\tilde{\tau}_{t}\bigg\|+\frac{1}{c'}\bigg\|\sum_{t=1}^{T}f(\btheta_{t})\tilde{\tau}_{t}-\sum_{t=1}^{T}f(\btheta_{t})\tau{}_{t}\bigg\|\nonumber \\
 & =\frac{1}{c'}\left|c'-c\right|\frac{1}{c}\bigg\|\sum_{t=1}^{T}f(\btheta_{t})\tilde{\tau}_{t}\bigg\|+\frac{1}{c'}\bigg\|\sum_{t=1}^{T}f(\btheta_{t})(\tilde{\tau}_{t}-\tau_{t})\bigg\|\nonumber \\
 & \le\frac{1}{c'}\left|\sum_{t=1}^{T}\tau_{t}-\sum_{t=1}^{T}\tilde{\tau}_{t}\right|\frac{1}{c}\bigg\|\sum_{t=1}^{T}f(\btheta_{t})\tilde{\tau}_{t}\bigg\|+\frac{1}{c'}\sum_{t=1}^{T}\|f(\btheta_{t})\|_{2}|\tilde{\tau}_{t}-\tau_{t}|\nonumber \\
 & \le\frac{1}{c'}\sum_{t=1}^{T}|\tilde{\tau}_{t}-\tau_{t}|\frac{1}{c}\bigg\|\sum_{t=1}^{T}f(\btheta_{t})\tilde{\tau}_{t}\bigg\|+\frac{K_{T}}{c'}\sum_{t=1}^{T}|\tilde{\tau}_{t}-\tau_{t}|,\label{eq:e_decom1}
\end{align}
where $K_{T}:=\max_{t=1,\ldots,T}\|f(\boldsymbol{\theta}_{t})\|_{2}$.
Consider $\frac{1}{c}\bigg\|\sum_{t=1}^{T}f(\btheta_{t})\tilde{\tau}_{t}\bigg\|$.
We can show that it is bounded by $K_{T}$ by
\begin{align*}
\frac{1}{c}\bigg\|\sum_{t=1}^{T}f(\btheta_{t})\tilde{\tau}_{t}\bigg\| & \le\frac{1}{c}\sum_{t=1}^{T}\|f(\btheta_{t})\|_{2}\tilde{\tau}_{t}\le\frac{K_{T}}{c}\sum_{t=1}^{T}\tilde{\tau}_{t}=K_{T}.
\end{align*}
Combining this bound with (\ref{eq:e_decom1}), we have
\begin{align}
\bigg\|\frac{1}{c}\sum_{t=1}^{T}f(\btheta_{t})\tilde{\tau}_{t}-\frac{1}{c'}\sum_{t=1}^{T}f(\btheta_{t})\tau{}_{t}\bigg\|_{2} & \le\frac{2K_{T}}{c'}\sum_{t=1}^{T}|\tilde{\tau}_{t}-\tau_{t}|\label{eq:cs_inq_before_poibin}
\end{align}
We will need to characterize the distribution of $|\tilde{\tau}_{t}-\tau_{t}|$.
Let $Z_{t}:=m_{t}-\nu_{t}$. By Lemma \ref{lemma:pdf_Z}, we have
\begin{align*}
p_{t}=\pr(\tilde{\tau}_{t}=1) & =\pr(Z_{t}>\rho_{t}-\epsilon_{abc})=1-F_{Z}(\rho_{t}-\epsilon_{abc})\\
 & =1-H[\rho_{t}-\epsilon_{abc}]+(2H[\rho_{t}-\epsilon_{abc}]-1)G_{b}(|\rho_{t}-\epsilon_{abc}|)\\
 & =\tau_{t}+(1-2\tau_{t})G_{b}(|\rho_{t}-\epsilon_{abc}|),
\end{align*}
where the decreasing function $G_{b}(x)\in(0,\frac{1}{2}]$ for any
$x\ge0$ is defined in Lemma \ref{lemma:pdf_Z}. We observe that $|\tilde{\tau}_{t}-\tau_{t}|\sim\mathrm{Bernoulli}(q_{t})$
where $q_{t}:=\pr(\tilde{\tau}_{t}\neq\tau_{t})=(1-p_{t})\tau_{t}+p_{t}(1-\tau_{t})$.
We can rewrite $q_{t}$ as
\begin{align*}
q_{t} & =\tau_{t}+p_{t}(1-2\tau_{t})\\
 & =\tau_{t}+\left[\tau_{t}+(1-2\tau_{t})G_{b}(|\rho_{t}-\epsilon_{abc}|)\right](1-2\tau_{t})\\
 & =G_{b}(|\rho_{t}-\epsilon_{abc}|).
\end{align*}
To prove the first claim, we take the expectation on both sides of
(\ref{eq:cs_inq_before_poibin}):
\begin{align*}
\mathbb{E}_{\tilde{\tau}_{1},\ldots,\tilde{\tau}_{T}}\bigg\|\frac{1}{c}\sum_{t=1}^{T}f(\btheta_{t})\tilde{\tau}_{t}-\frac{1}{c'}\sum_{t=1}^{T}f(\btheta_{t})\tau{}_{t}\bigg\|_{2} & \le\frac{2K_{T}}{c'}\mathbb{E}_{\tilde{\tau}_{t}}\left[\sum_{t=1}^{T}|\tilde{\tau}_{t}-\tau_{t}|\right]\\
 & =\frac{2K_{T}}{c'}\mu_{T},
\end{align*}
where $\mu_{T}=\mathbb{E}_{\tilde{\tau}_{t}}\left[\sum_{t=1}^{T}|\tilde{\tau}_{t}-\tau_{t}|\right]=\sum_{t=1}^{T}G_{b}(|\rho_{t}-\epsilon_{abc}|)$
and we use the fact that $\mathbb{E}_{\tilde{\tau}_{t}}|\tilde{\tau}_{t}-\tau_{t}|=q_{t}$.
Note that these are $T$ independent, marginal expectations i.e.,
not conditioning on noise added to the ABC threshold.

\textbf{Proof of the second claim} Observe that $G_{b}(|\rho_{t}-\epsilon_{abc}|)\to0$
as $b\to0$. The claim follows by noting that $b=O(1/N)$.

\textbf{Proof of the third claim} Based on (\ref{eq:cs_inq_before_poibin}),
characterizing the tail bound of $\bigg\|\frac{1}{c}\sum_{t=1}^{T}f(\btheta_{t})\tilde{\tau}_{t}-\frac{1}{c'}\sum_{t=1}^{T}f(\btheta_{t})\tau{}_{t}\bigg\|_{2}$
amounts to establishing a tail bound on $S_{T}:=\sum_{t=1}^{T}|\tilde{\tau}_{t}-\tau_{t}|$.
By the Markov's inequality,
\begin{align*}
P(S_{T}\le s) & \ge1-\mathbb{E}[S_{T}]/s\\
 & =1-\frac{1}{s}\sum_{t=1}^{T}G_{b}(|\rho_{t}-\epsilon_{abc}|)\\
 & =1-\frac{1}{s}\sum_{t=1}^{T}\frac{1}{6}\left[4\exp\left(-\frac{|\rho_{t}-\epsilon_{abc}|}{2b}\right)-\exp\left(-\frac{|\rho_{t}-\epsilon_{abc}|}{b}\right)\right]\\
 & \ge1-\frac{2}{3s}\sum_{t=1}^{T}\exp\left(-\frac{|\rho_{t}-\epsilon_{abc}|}{2b}\right).
\end{align*}

Applying this bound to (\ref{eq:cs_inq_before_poibin}) gives
\[
P\left(\frac{2K_{T}}{c'}S_{T}\le\frac{2K_{T}}{c'}s\right)=P(S_{T}\le s).
\]
With a reparametrization $a:=\frac{2K_{T}}{c'}s$ so that $s=\frac{ac'}{2K_{T}},$we
have
\[
P\left(\frac{2K_{T}}{c'}S_{T}\le a\right)\ge1-\frac{4K_{T}}{3ac'}\sum_{t=1}^{T}\exp\left(-\frac{|\rho_{t}-\epsilon_{abc}|}{2b}\right),
\]
Since $\bigg\|\frac{1}{c}\sum_{t=1}^{T}f(\btheta_{t})\tilde{\tau}_{t}-\frac{1}{c'}\sum_{t=1}^{T}f(\btheta_{t})\tau{}_{t}\bigg\|_{2}\le\frac{2K_{T}}{c'}S_{T}$,
we have $P\bigg(\bigg\|\frac{1}{c}\sum_{t=1}^{T}f(\btheta_{t})\tilde{\tau}_{t}-\frac{1}{c'}\sum_{t=1}^{T}f(\btheta_{t})\tau{}_{t}\bigg\|_{2}\le a\bigg)\ge P\left(\frac{2K_{T}}{c'}S_{T}\le a\right)$.
which gives the result in the third claim.

\end{proof}

\end{document}

%% file: notation.tex
\newcommand{\punt}[1]{}

\newtheorem{thm}{Theorem}[section]
\newtheorem{lem}{Lemma}[section]
\newtheorem{cor}[thm]{Corollary}
\newtheorem{prop}[thm]{Proposition}
\newtheorem{defn}{Definition}[section]

\newtheorem{remark}{Remark}

\newtheorem{condition}{Condition}
\def\argmax{\mathop{\rm arg\,max}}
\def\argmin{\mathop{\rm arg\,min}}
\newcommand{\transpose}{{\scriptscriptstyle T}}
\newcommand{\trp}{{^\top}} % transpose
\newcommand{\reals}{\mathbb{R}}
\newcommand{\biggram}{\overrightarrow{\mathbf{K}}}
\newcommand{\biggramsep}{{\vK \otimes \vL}}
\newcommand{\veckernel}{\overrightarrow{k}}
\newcommand{\chvec}{\ch_{\veckernel}}
\newcommand{\chvecsep}{\ch_{k\vL}}
\newcommand{\chdict}{\ch_{\cd}}
\newcommand{\expect}{\mathbb{E}}
\newcommand{\integers}{\mathbf{Z}}
\newcommand{\naturals}{\mathbf{N}}
\newcommand{\rationals}{\mathbf{Q}}

\newcommand{\ca}{\mathcal{A}}
\newcommand{\cb}{\mathcal{B}}
\newcommand{\cc}{\mathcal{C}}
\newcommand{\cd}{\mathcal{D}}
\newcommand{\ce}{\mathcal{E}}
\newcommand{\cf}{\mathcal{F}}
\newcommand{\cg}{\mathcal{G}}
\newcommand{\ch}{\mathcal{H}}
\newcommand{\ci}{\mathcal{I}}
\newcommand{\cj}{\mathcal{J}}
\newcommand{\ck}{\mathcal{K}}
\newcommand{\cl}{\mathcal{L}}
\newcommand{\cm}{\mathcal{M}}
\newcommand{\cn}{\mathcal{N}}
\newcommand{\co}{\mathcal{O}}
\newcommand{\cp}{\mathcal{P}}
\newcommand{\cq}{\mathcal{Q}}
\newcommand{\calr}{\mathcal{R}}
\newcommand{\cs}{\mathcal{S}}
\newcommand{\ct}{\mathcal{T}}
\newcommand{\cu}{\mathcal{U}}
\newcommand{\cv}{\mathcal{V}}
\newcommand{\cw}{\mathcal{W}}
\newcommand{\cx}{\mathcal{X}}
\newcommand{\cy}{\mathcal{Y}}
\newcommand{\cz}{\mathcal{Z}}
\newcommand{\cS}{\mathcal{S}}
\newcommand{\pr}{\mathbb{P}}
\newcommand{\predsp}{\cy}  %{\hat{\cy}}
\newcommand{\outsp}{\cy}

\newcommand{\prxy}{P_{\cx \times \cy}}
\newcommand{\prx}{P_{\cx}}
\newcommand{\prygivenx}{P_{\cy\mid\cx}}
\newcommand{\ex}{\mathbb{E}}
\newcommand{\cov}{\textrm{Cov}}
\newcommand{\kl}{\textrm{KL}}
\newcommand{\law}{\mathcal{L}}
\newcommand{\as}{\textrm{ a.s.}}
\newcommand{\io}{\textrm{ i.o.}}
\newcommand{\ev}{\textrm{ ev.}}
\newcommand{\convd}{\stackrel{d}{\to}}
\newcommand{\eqd}{\stackrel{d}{=}}
\newcommand{\del}{\nabla}
\newcommand{\loss}{V}
\newcommand{\risk}{R}
\newcommand{\emprisk}{\hat{R}_{\ell}}
\newcommand{\lossfnl}{\risk}
\newcommand{\emplossfnl}{\emprisk}
\newcommand{\empminimizer}[1]{\hat{#1}_{\ell}} 
\newcommand{\empminimizerfa}{\hat{f}_{\ell}^{1}}
\newcommand{\empminimizerfb}{\hat{f}_{\ell}^{2}}
\newcommand{\minimizer}[1]{#1_{*}} 
\newcommand{\etal}{\textrm{et. al.}}
\newcommand{\rademacher}[1]{\calr_{#1}}
\newcommand{\emprademacher}[1]{\hat{\calr}_{#1}}

\newcommand{\trace}{\operatorname{trace}}
\newcommand{\rank}{\text{rank}}
\newcommand{\linspan}{\text{span}}
\newcommand{\spn}{\text{span}}
\newcommand{\proj}{\text{Proj}}
\def\argmax{\mathop{\rm arg\,max}}
\def\argmin{\mathop{\rm arg\,min}}

\newcommand{\bfx}{\mathbf{x}}
\newcommand{\bfy}{\mathbf{y}}
\newcommand{\bfl}{\bm{\lambda}}
\newcommand{\bfm}{\mathbf{\mu}}
\newcommand{\calL}{\mathcal{L}}

\newcommand{\vX}{\mathbf{X}}
\newcommand{\vY}{\mathbf{Y}}
\newcommand{\vA}{\mathbf{A}}
\newcommand{\vB}{\mathbf{B}}
\newcommand{\vE}{\mathbf{E}}
\newcommand{\vK}{\mathbf{K}}
\newcommand{\vD}{\mathbf{D}}
\newcommand{\vU}{\mathbf{U}}
\newcommand{\vL}{\mathbf{L}}
\newcommand{\vI}{\mathbf{I}}
\newcommand{\vC}{\mathbf{C}}
\newcommand{\vV}{\mathbf{V}}
\newcommand{\vM}{\mathbf{M}}
\newcommand{\vN}{\mathbf{N}}
\newcommand{\vQ}{\mathbf{Q}}
\newcommand{\vR}{\mathbf{R}}
\newcommand{\vS}{\mathbf{S}}
\newcommand{\vT}{\mathbf{T}}
\newcommand{\vZ}{\mathbf{Z}}
\newcommand{\vW}{\mathbf{W}}
\newcommand{\vH}{\mathbf{H}}
\newcommand{\vsig}{\bm{\Sigma}}
\newcommand{\vLam}{\bm{\Lambda}}
\newcommand{\vLambda}{\bm{\Lambda}}
\newcommand{\vlam}{\bm{\lambda}}

\newcommand{\vw}{\mathbf{w}}
\newcommand{\vx}{\mathbf{x}}
\newcommand{\vxi}{\bm{\xi}}     
\newcommand{\valpha}{\bm{\alpha}}
\newcommand{\vbeta}{\bm{\beta}}
\newcommand{\veta}{\bm{\eta}}
\newcommand{\vsigma}{\bm{\sigma}}
\newcommand{\vepsilon}{\bm{\epsilon}}
\newcommand{\vdelta}{\bm{\delta}}
\newcommand{\vnu}{\bm{\nu}}
\newcommand{\vd}{\mathbf{d}}
\newcommand{\vs}{\mathbf{s}}
\newcommand{\vt}{\mathbf{t}}
\newcommand{\vh}{\mathbf{h}}
\newcommand{\ve}{\mathbf{e}}
\newcommand{\vf}{\mathbf{f}}
\newcommand{\vg}{\mathbf{g}}
\newcommand{\vz}{\mathbf{z}}
\newcommand{\vm}{\mathbf{m}}
\newcommand{\vk}{\mathbf{k}}
\newcommand{\va}{\mathbf{a}}
\newcommand{\vb}{\mathbf{b}}
\newcommand{\vv}{\mathbf{v}}
\newcommand{\vr}{\mathbf{r}}
\newcommand{\vy}{\mathbf{y}}
\newcommand{\vu}{\mathbf{u}}
\newcommand{\vp}{\mathbf{p}}
\newcommand{\vq}{\mathbf{q}}
\newcommand{\vc}{\mathbf{c}}
\newcommand{\vn}{\mathbf{n}}
\newcommand{\vP}{\mathbf{P}}
\newcommand{\vG}{\mathbf{G}}

\newcommand{\hil}{\ch}
\newcommand{\rkhs}{\hil}
\newcommand{\manifold}{\cm} 
\newcommand{\cloud}{\cc} 
\newcommand{\graph}{\cg}    
\newcommand{\vertices}{\cv} 
\newcommand{\coreg}{{\scriptscriptstyle \cc}}
\newcommand{\intrinsic}{{\scriptscriptstyle \ci}}
\newcommand{\ambient}{{\scriptscriptstyle \ca}}
\newcommand{\isintrinsic}[1]{#1^{\scriptscriptstyle \ci}}
\newcommand{\isambient}[1]{#1^{\scriptscriptstyle \ca}}
\newcommand{\hilamb}{\ch^\ambient}
\newcommand{\hilintr}{\ch^\intrinsic}
\newcommand{\M}{\text{FIX ME NOW}}  %temporary
\newcommand{\ktilde}{\tilde{k}}
\newcommand{\corkhs}{\tilde{\hil}}
\newcommand{\cok}{\ktilde}
\newcommand{\coK}{\tilde{K}}
\newcommand{\copar}{\lambda}
\newcommand{\reg}{\gamma}
\newcommand{\rega}{\gamma_{1}}
\newcommand{\regb}{\gamma_{2}}
\newcommand{\regamb}{\gamma_{\ambient}}
\newcommand{\regintr}{\gamma_{\intrinsic}}
\newcommand{\regfn}{\Omega}
\newcommand{\regfnintr}{\regfn_{\intrinsic}}
\newcommand{\regfnamb}{\regfn_{\ambient}}
\newcommand{\regfncoreg}{\regfn_{\coreg}}
\newcommand{\bq}{\begin{equation}}
\newcommand{\eq}{\end{equation}}
\newcommand{\ba}{\begin{eqnarray}}
\newcommand{\ea}{\end{eqnarray}}
\newcommand{\spana}{\cl^{1}}
\newcommand{\spanb}{\cl^{2}}
\newcommand{\ha}{\hil^{1}}
\newcommand{\hb}{\hil^{2}}
\newcommand{\fa}{f^{1}}
\newcommand{\fb}{f^{2}}
\newcommand{\Fa}{\cf^{1}}
\newcommand{\Fb}{\cf^{2}}
\newcommand{\ka}{k^{1}}
\newcommand{\kb}{k^{2}}
\newcommand{\vcok}{\boldsymbol{\cok}}
\newcommand{\vka}{\vk^{1}}
\newcommand{\vkb}{\vk^{2}}
\newcommand{\Ka}{K^{1}}
\newcommand{\Kb}{K^{2}}
\newcommand{\kux}{\vk_{Ux}}
\newcommand{\kuz}{\vk_{Uz}}
\newcommand{\kuu}{K_{UU}}
\newcommand{\kul}{K_{UL}}
\newcommand{\klu}{K_{LU}}
\newcommand{\kll}{K_{LL}}
\newcommand{\kuxa}{\vk_{Ux}^{1}}
\newcommand{\kuza}{\vk_{Uz}^{1}}
\newcommand{\Kuua}{K_{UU}^{1}}
\newcommand{\Kula}{K_{UL}^{1}}
\newcommand{\Klua}{K_{LU}^{1}}
\newcommand{\Klla}{K_{LL}^{1}}
\newcommand{\kuxb}{\vk_{Ux}^{2}}
\newcommand{\kuzb}{\vk_{Uz}^{2}}
\newcommand{\Kuub}{K_{UU}^{2}}
\newcommand{\Kulb}{K_{UL}^{2}}
\newcommand{\Klub}{K_{LU}^{2}}
\newcommand{\Kllb}{K_{LL}^{2}}
\newcommand{\Ksum}{S}
\newcommand{\ksum}{s}
\newcommand{\vksum}{\vs}
\newcommand{\intrinsicRegMat}{M_{\intrinsic}}
\newcommand{\coregPointCloudMat}{M_{\coreg}}
\newcommand{\vid}[1]{#1^\text{vid}}
\newcommand{\aud}[1]{#1^\text{aud}}
\newcommand{\bad}[1]{#1_\text{bad}}
\newcommand{\empcompat}{\hat{\chi}}
\newcommand{\nn}{\ensuremath{k}}
\newcommand{\uva}{\ushort{\va}}
\newcommand{\uvf}{\ushort{\vf}}
\newcommand{\uvg}{\ushort{\vg}}
\newcommand{\uvk}{\ushort{\vk}}
\newcommand{\uvw}{\ushort{\vw}}
\newcommand{\uvh}{\ushort{\vh}}
\newcommand{\uvbeta}{\ushort{\vbeta}}
\newcommand{\uA}{\ushort{A}}
\newcommand{\uG}{\ushort{G}}
% For Vikas :)
\def\la{{\langle}}
\def\ra{{\rangle}}
\def\R{{\reals}}
%added by AK
\newcommand{\mbf}[1]{\mathbf{#1}}
\newcommand{\mbb}[1]{\mathbb{#1}}
\newcommand{\mcal}[1]{\mathcal{#1}}
\newcommand{\remove}[1]{}
\newcommand{\nuc}[1]{\left\lVert #1\right\rVert_*}
\newcommand{\spec}[1]{\left\lVert #1\right\rVert_2}
\newcommand{\frob}[1]{\left\lVert #1\right\rVert_\fro}
\newcommand{\norm}[1]{\left\lVert #1\right\rVert}
\newcommand{\abs}[1]{\left\lvert #1\right\rvert}
\newcommand{\obs}[1]{P_\Omega({#1})}
\newcommand{\red}[1]{{\color{red} #1}}
%\newenvironment{psmallmatrix}{\left(\begin{smallmatrix}} {\end{smallmatrix}\right)}
%%%%%%%%%%%%%%%%%
\newcommand{\ball}{\mathcal{B}}
\newcommand{\sphere}{\mathcal{S}}
\newcommand{\X}{\mathcal{X}}
\newcommand{\domain}{\mathcal{C}}
\newcommand{\polytope}{\mathcal{P}}
\newcommand{\Am}{\bm{A}}
\newcommand{\Bm}{\bm{B}}
\newcommand{\xv}{\bm{x}}
\newcommand{\yv}{\bm{y}}
\newcommand{\zv}{\bm{z}}
\newcommand{\sv}{\bm{s}}
\newcommand{\av}{\bm{a}}
\newcommand{\bv}{\bm{b}}
\newcommand{\pv}{\bm{p}}
\newcommand{\vo}{\bm{o}}
\newcommand{\uv}{\bm{u}}
\newcommand{\wv}{\bm{w}}
\newcommand{\omegav}{\bm{\omega}}
\newcommand{\mv}{\bm{m}}
\newcommand{\GW}{G_\setC}%{G_{\!f,\setC}}
\newcommand{\G}{G}
\newcommand{\0}{\bm{0}}
\newcommand{\id}{\bm{\iota}}
\newcommand{\alphav}{\bm{\alpha}}
\newcommand{\nuv}{\bm{\nu}}
\newcommand{\thetav}{\bm{\theta}}
\newcommand{\lambdav}{\bm{\lambda}}
\newcommand{\epsilonv}{\bm{\epsilon}}
\newcommand{\row}{\text{row}}
\newcommand{\col}{\text{col}}
\newcommand{\lft}{\text{left}}
\newcommand{\rgt}{\text{right}}
\newcommand{\one}{\mathbf{1}} % all one vector
\newcommand{\Nrm}{\mathcal{N}}
\newcommand{\Em}{\mathbb{E}}
\newcommand{\Var}{\mathbb{V}}
\newcommand{\Tr}{\mathrm{Tr}}
\newcommand{\diag}{\mathrm{diag}}
% \DeclareMathOperator*{\argmin}{arg\,min}
% \DeclareMathOperator*{\argmax}{arg\,max}
% ----- Math definitions -------------------
% \newcommand{\vx}{\mathbf{x}}
% \newcommand{\var}{\mathrm{var}}
% \newcommand{\Dat}{\mathcal{D}}
% \DeclareMathOperator{\diam}{diam}
% \DeclareMathOperator{\diag}{diag}
% \DeclareMathOperator{\dom}{dom}         % domain
%\newcommand{\todo}[1]{\marginpar[\hspace*{4.5em}\textbf{TODO}\hspace*{-4.5em}]{\textbf{TODO}}\textbf{TODO:} #1}
\newcommand{\aatop}[2]{\genfrac{}{}{0pt}{}{#1}{#2}}
\newcommand{\ie}{\textit{i}.\textit{e}.}
\newcommand{\eg}{\textit{e}.\textit{g}.}
% -----  some definitions ------------------
\newcommand{\inv}{^{-1}}
\newcommand{\onehlf}{\frac{1}{2}}
\newcommand{\tonehlf}{\tfrac{1}{2}}
\newcommand{\defnref}[1]{Def.~\ref{defn:#1}}
\newcommand{\propref}[1]{Prop.~\ref{prop:#1}}
\newcommand{\figref}[1]{Fig.~\ref{fig:#1}}  % use for citing figs
\newcommand{\secref}[1]{Sec.~\ref{sec:#1}}  % use for citing secs
\newcommand{\subsecref}[1]{Sec.~\ref{subsec:#1}}  % use for citing subsecs
\newcommand{\suppsecref}[1]{Supplementary Sec.~\ref{supp:#1}}  % use for citing secs in supp
\newcommand{\tabref}[1]{Table ~\ref{tab:#1}}  % use for citing tables
\newcommand{\algoref}[1]{Algorithm~\ref{algo:#1}}  % use for citing algorithms
\newcommand{\corref}[1]{Corr.~\ref{cor:#1}}  % use for citing Corr
\newcommand{\thmref}[1]{Thm.~\ref{thm:#1}}  % use for citing Theorems
\newcommand{\lemref}[1]{Lemma.~\ref{lemma:#1}}  % use for citing lemmas
% ----- Math definitions -------------------
\newcommand{\var}{\mathrm{var}}
\newcommand{\Dat}{\mathcal{D}}
\newcommand{\Ev}{\mathcal{E}}
\newcommand{\vxloc}{\mathcal{X}}
\newcommand{\vkx}{\mathbf{k_x}}
\newcommand{\vky}{\mathbf{k_y}}
\newcommand{\vphi}{\mathbf{\ensuremath{\bm{\phi}}}}
\newcommand{\bmu}{\mathbf{\ensuremath{\bm{\mu}}}}
\newcommand{\bK}{\mathbf{K}}
\newcommand{\Vm}{\mathbb{V}} 
\newcommand{\vone}{\mathbf{1}} % one vector
\newcommand{\fmap}{\bm{\lambda}{_{map}}}
\newcommand{\bphi}{\bm{\phi}}
\newcommand{\phimap}{{\hat {\bm{\phi}}}{_{map}}}
\newcommand{\fmapstar}{\mathbf{f^*_{map}}}
\newcommand{\muf}{\ensuremath{\mu_f}}
\newcommand{\vmuf}{\mathbf{\ensuremath{\bm{\mu}_f}}}
\newcommand{\vpi}{\mathbf{\ensuremath{\bm{\pi}}}}
\newcommand{\vmu}{\mathbf{\ensuremath{\bm{\mu}}}}
\newcommand{\vtheta}{\mathbf{\ensuremath{\bm{\theta}}}}
\newcommand{\LL}{\ensuremath{\mathcal{L}}}
\newcommand{\mR}{\mathcal{R}}
\newcommand{\mW}{\mathbf{W}}
\newcommand{\mZ}{\mathbf{Z}}
\newcommand{\mC}{\mathbf{C}}
\newcommand{\mJ}{\mathbf{J}}
\newcommand{\set}[1]{\{#1\}}
\newcommand{\kml}{{\hat \vk}_{ML}}
\newcommand{\Lprior}{\Lambda_p}
\newcommand{\Lix}{L_{x}}
\newcommand{\thetmap}{\theta_{\mu}}
\newcommand{\Lmap}{\Lambda_\mu}
\newcommand{\thetli}{\theta_{l}}
\newcommand{\Lli}{\Lambda_{l}}
\newcommand{\nsevar}{\sigma^2}
\newcommand{\nsestd}{\sigma}
\newcommand{\vchi}{\bm{\chi}}
\newcommand{\vomega}{\bm{\omega}}

\newcommand{\mecrej}{\mathcal{M}_{rej}}
% ---- For commenting -----------
%\usepackage{xcolor}  
% \definecolor{gray}{rgb}{0.5, .5, .5}
\newcommand{\mj}[1]{{\color{blue}{ mijung : #1}}}